\documentclass{article}

\usepackage{arxiv}
\usepackage{amsmath}
\usepackage[utf8]{inputenc} 
\usepackage[T1]{fontenc}    
\usepackage{hyperref}       
\usepackage{url}            
\usepackage{booktabs}       
\usepackage{amsfonts}       
\usepackage{nicefrac}       
\usepackage{microtype}      
\usepackage{graphicx}
\usepackage{natbib}

\usepackage{amssymb}
\usepackage{amsfonts}
\usepackage{amsthm}
\usepackage{bm}

\theoremstyle{remark}
\newtheorem*{remark}{Remark}
\theoremstyle{plain}
\newtheorem*{lemma}{Lemma}
\newtheorem*{corollary}{Corollary}
\newtheorem*{interpretation}{Interpretation}
\newtheorem*{numericalanchor}{Numericalanchor}

\theoremstyle{plain}
\newtheorem{mainthm}{Theorem}
\newtheorem{mainlem}{Lemma}

\title{A Control Theory of Predictability\\ in Latent World Models}

\author{
    Hanzhe You\textsuperscript{1}, \ Yonggang Zhang\textsuperscript{2}, \ Maohao Ran\textsuperscript{3}, \ Zhiqin Yang\textsuperscript{2}, \ Zhenyuan Zhang\textsuperscript{2}, \ Wei Xue\textsuperscript{2} \\[3pt]
    \textbf{Jun Song\textsuperscript{3,*}, \quad Xinmei Tian\textsuperscript{1,*}, \quad Yike Guo\textsuperscript{2,*}} \\[8pt]
    {\normalsize \textsuperscript{1}University of Science and Technology of China} \\[1pt]
    {\normalsize \textsuperscript{2}The Hong Kong University of Science and Technology, HKGAI} \\[1pt]
    {\normalsize \textsuperscript{3}Hong Kong Baptist University}
}

\renewcommand{\shorttitle}{A Control Theory of Predictability in Latent World Models}

\hypersetup{
pdftitle={A Control Theory of Predictability in Latent World Models},
pdfsubject={cs.LG},
pdfauthor={Hanzhe You, Yonggang Zhang, Maohao Ran, Zhiqin Yang, Zhenyuan Zhang, Wei Xue, Jun Song, Xinmei Tian, Yike Guo},
pdfkeywords={World Models, Joint-Embedding Prediction, Koopman Operators, Pseudospectra, Model-Predictive Control},
}

\newcommand\blfootnote[1]{%
  \begingroup\renewcommand\thefootnote{}\footnotetext{#1}\endgroup%
}

\begin{document}
\maketitle
\blfootnote{\textsuperscript{*}\,Corresponding authors.}

\begin{abstract}
Latent world models are trained to predict future states in a learned representation and are then
deployed inside a planner that selects actions by simulating them forward. Current practice adopts
the prediction error, the single- or multi-step rollout loss on held-out data, as the training and
model-selection objective, on the assumption that a lower prediction error yields better control. We
show that this assumption is unreliable for a structural reason: a planner does not query the model
on the training distribution but on the states that its candidate actions reach, which generally
leave the data manifold, so an error averaged over the data cannot by itself govern control. We
therefore reframe the objective as the discrepancy between the predicted and the true plan-cost at
the plan the planner commits to, and prove that the planner's suboptimality is bounded by twice this
discrepancy, whereas the data-averaged prediction error neither bounds nor tracks it. Under a
linear-control premise the discrepancy separates into two terms. The first is a small on-manifold
residual, on which the predicted and true dynamics agree and which a spectral tax prices through the
non-normality of the latent transition operator. The second is an off-manifold divergence, on which
an action carries the state off the manifold and the two dynamics diverge; this divergence is the
binding term and is bounded by no data-averaged error. Synthetic operators confirm the pricing
formulas, and latent model-predictive control experiments confirm the decoupling: across seeds, the
single-step validation error is essentially uncorrelated with control success, whereas a fidelity
score on the planner-reachable measure tracks it.
\end{abstract}


\bigskip

\section{Introduction}
\label{sec:intro}

A latent world model supports planning by predicting the future of a compact latent state, which
allows an agent to act without modeling every detail of its observations. This places two competing
requirements on the representation: it must retain enough of the signal to be useful, and it must
remain \emph{predictable} from its own past by a simple predictor acting in the latent space.
Joint-Embedding Predictive Architectures \citep{assran2023self,lecun2022path} make this trade-off the
explicit training objective, mapping observations to a latent space with a shared encoder and
operating a predictor inside that space rather than reconstructing the input. Once trained, the
latent model is deployed inside a planner, such as model-predictive control or trajectory
optimization in latent space \citep{hafner2020dream,garcia2013model,assran2025vjepa2}, that rolls
candidate action sequences forward and executes the one of lowest predicted cost.

A world model is evaluated not by its prediction error but by the quality of the actions it induces,
and here common practice rests on an assumption that we argue is unsound: that lowering the
prediction error, the single- or multi-step rollout loss on held-out data, is the correct objective
for control. A planner does not query the model on the training distribution. It queries it on the
states that candidate actions reach, and those states generally leave the data manifold, where the
model extrapolates. The prediction loss is an average over the data distribution, whereas control
success is a functional of the states the planner visits; the two need not move together, and we show
that they need not.

This paper develops a control theory for latent world models organized around this objective. We
prove in Section~\ref{sec:control-theory} that the planner's suboptimality is controlled by a single
scalar, the largest discrepancy between the model's \emph{predicted} plan-cost and the \emph{true}
plan-cost over the plans the planner considers, so that the target for control is to make the
predicted cost track the true cost rather than to lower the averaged prediction error. We then
establish why the averaged error is the wrong instrument. It cannot bound this discrepancy once the
query leaves the data support, which is the decoupling; the harmful part of the discrepancy is
one-sided, arising from a \emph{hole} on which the predicted cost underestimates the true; and the
discrepancy is seeded only when the predictor carries realizable nonlinearity. To make the objective
quantitative, Section~\ref{sec:pricing} prices the predicted-versus-true gap. Under a linear-control
premise it separates into a small on-manifold residual, on which the predicted and true dynamics
agree, and an off-manifold divergence, on which an action carries the state off the data manifold and
the two dynamics diverge. The off-manifold divergence is the binding term and the object of our main
claim; it is priced by a \emph{rollout tax} that compounds the residual along the action-selected
trajectory and is bounded by no data-average. The on-manifold residual is priced by a \emph{spectral
tax} set by the non-normality of the latent transition operator, whose closed-form frontier vanishes
on the self-adjoint, Gaussian boundary where prior identifiability guarantees live and thereby places
those guarantees at the zero-cost corner of a broader surface; it does not by itself govern control.

\paragraph{Contributions.}
\begin{enumerate}
    \item[\textbf{(1)}] \textbf{The correct control objective, and one core theorem
    (Section~\ref{sec:control-theory}).} We prove that a latent planner's suboptimality is bounded
    by twice the supremum gap between predicted and true plan-cost over the candidate set. The
    control objective is therefore to make predicted cost track true cost at the committed plan;
    we then prove three supporting lemmas showing why the data-averaged prediction error cannot
    serve this objective: it fails to bound the gap once the query leaves the data support, it is
    two-sided where the gap is one-sided, and it is blind to the realizable nonlinearity that seeds
    the gap.

    \item[\textbf{(2)}] \textbf{Pricing the gap, with the off-manifold divergence as the binding term
    (Section~\ref{sec:pricing}).} Under a linear-control premise the gap separates into an on-manifold
    residual and an off-manifold divergence. We price the divergence---an off-manifold rollout tax
    that no data-average bounds---and combine it with the residual into a single control-error theorem
    with an optimal exploration width. The on-manifold residual is priced by a spectral tax that
    vanishes on the self-adjoint boundary of prior work but, as the experiments show, does not by
    itself govern control.

    \item[\textbf{(3)}] \textbf{A theory-guided method and empirical validation
    (Section~\ref{sec:experiments}).} The analysis prescribes a minimal intervention, a linear state
    readout, together with three falsifiable predictions. Latent model-predictive control experiments
    confirm the decoupling, in that the single-step validation error does not separate a
    fourteen-point spread in success across seeds, and identify the off-manifold amplification as an
    informative training signal; synthetic-operator experiments corroborate the spectral-tax pricing
    formulas and the off-manifold suboptimality bound.
\end{enumerate}

\noindent Full derivations and proofs of every result, together with the numerical-verification
protocols, are collected in the appendices.

\section{Related Work}
\label{sec:related}

\paragraph{Self-supervised representation and JEPA.}
Joint-embedding methods learn representations by predicting one view from another in latent space
\citep{assran2023self,lecun2022path}, with anti-collapse regularizers preventing trivial solutions
\citep{bardes2024revisiting,tian2021understanding,balestriero2022contrastive}. Recent theory
characterizes the ideal case: LeJEPA shows that isotropic Gaussian embeddings minimize downstream
worst-case risk \citep{balestriero2025lejepa}, and that in an Ornstein--Uhlenbeck environment the
same distribution yields linear identifiability of the latent factors \citep{klindt2026does}. These
guarantees rest on an idealized latent transition---self-adjoint, reversible, Gaussian. We take
that ideal case as the zero-cost boundary of a broader cost surface and price the non-normal
regime outside it.

\paragraph{World models and latent planning.}
Latent world models roll a learned dynamics forward and plan by trajectory optimization or
model-predictive control \citep{hafner2020dream,garcia2013model}, a paradigm now scaled to video
\citep{assran2025vjepa2}. Our analysis concerns exactly this deployment: the planner's query
measure, not the data measure, governs success, and it is the discrepancy between predicted and
true reachable dynamics---not the training loss---that prices control quality.

\paragraph{Koopman operators and spectral estimation.}
Casting nonlinear dynamics through the linear Koopman operator underlies a large body of
data-driven modeling \citep{brunton2016koopman,mezic2005spectral,lusch2018deep,mardt2018vampnets,noe2013variational,otto2019linearly},
with recent attention to rigorous spectral computation \citep{colbrook2023residual,colbrook2024rigorous}.
We use the operator's pseudospectrum, rather than its spectrum, as the pricing object, following
the non-normal-matrix viewpoint of \citet{trefethen2020spectra}.

\paragraph{Rollout, exposure bias, and slow features.}
The mismatch between teacher-forced training and self-fed inference---exposure bias---is
classically addressed by scheduled sampling and related schemes
\citep{bengio2015scheduled,venkatraman2015improving}. We show this fix is orthogonal to the
off-manifold holes that govern control. Slow-feature analysis treats slowness as a proxy for
predictive usefulness \citep{wiskott2002slow,sprekeler2011relation}; a refinement of this view
under multi-step rollout (high-content slow features are the most rollout-fragile) is developed
in Appendix~\ref{sec:module-R}.

\section{A Control Theory for Latent-Space Planning}
\label{sec:control-theory}

Standard practice measures a world model by its prediction error and treats that error as a proxy
for control quality. We open with the setup that separates the two quantities, then state the one
theorem that identifies the correct control objective, and finally prove three lemmas that
explain why the prediction error cannot serve it.

\subsection{Setup}
\label{ssec:setup}
Let $\{X_t\}$ be a stationary, ergodic, first-order Markov process with transition (Koopman)
operator $P$ on $L^2(\mu)$; restricting to the mean-zero subspace $H_0=\mathbf 1^\perp$ gives
$P_0=P|_{H_0}$, real with $\|P_0\|\le1$. A whitened encoder corresponds to an isometry
$V:\mathbb R^d\to H_0$; the \emph{content} it retains is $\mathcal C=\|P_0V\|_{\mathrm{HS}}^2$ and
the single-step \emph{gap} is $\delta^2=\|(I-\Pi)P_0V\|^2$ with $\Pi=VV^*$. For control we take the
deployed latent dynamics to be the learned linear model
\begin{equation}
    z_{t+1}=\hat A z_t+\hat B u_t+\varepsilon_t,
    \label{eq:dyn}
\end{equation}
with $(\hat A,\hat B)$ the least-squares fit and residual $\varepsilon$. Let
$\mathcal M=\mathrm{supp}\,\mu$ be the data manifold, $r(z)=\mathrm{dist}(z,\mathcal M)$ the
off-manifold distance, and assume the residual grows linearly off the manifold,
$\|\varepsilon(z,u)\|\le\delta_0+L\,r(z)$, with $\delta_0$ the on-manifold residual and $L$ a
nonlinearity scale ($L=0$ for a genuinely linear system).

A planner ranks candidate action sequences $u\in\mathcal U$ by a predicted terminal cost and
executes the best. Write the predicted and true terminal costs to a goal $z_g$ as
\begin{equation}
    \hat D(u)=\|\hat z_H(u)-z_g\|,\qquad D(u)=\|z_H(u)-z_g\|,
    \label{eq:costs}
\end{equation}
where $\hat z_H(u)$ is the rollout of \eqref{eq:dyn} under $u$ and $z_H(u)$ the true terminal
state. The planner selects $\hat u=\arg\min_u\hat D(u)$, while the true optimum is
$u^\star=\arg\min_u D(u)$. Because candidate actions reach states off $\mathcal M$, the planner
queries the predictor on a \emph{reachable measure} $\nu$ its candidates generate, generally with
$\mathrm{supp}\,\nu\supsetneq\mathrm{supp}\,\mu$. Reading the same residual under the two measures
gives the training gap $\delta_\mu^2=\mathbb E_\mu\|\varepsilon\|^2$ (the monitored MSE) and the
deployment gap $\delta_\nu^2=\mathbb E_\nu\|\varepsilon\|^2$ (what success rests on).

\subsection{The control objective: predicted cost must track true cost}
\label{ssec:core}

Everything the planner does is choose $\hat u$ to minimize $\hat D$. Its regret is therefore
controlled by how faithfully $\hat D$ ranks candidates against $D$---nothing else. This is the
content of the core theorem.

\begin{mainthm}[Control objective: suboptimality is priced by the predicted-versus-true gap]
\label{thm:core}
For a planner that executes $\hat u=\arg\min_{u\in\mathcal U}\hat D(u)$,
\begin{equation}
    D(\hat u)-D(u^\star)\ \le\ 2\sup_{u\in\mathcal U}\big|\hat D(u)-D(u)\big|.
    \label{eq:core}
\end{equation}
Consequently the sole objective for control is to make the predicted plan-cost $\hat D$ agree with
the true plan-cost $D$ on the candidate set---in particular at the plan $\hat u$ the planner
commits to. Lowering the data-averaged prediction error is neither necessary nor sufficient.
\end{mainthm}
\begin{proof}
Let $g=\sup_u|\hat D(u)-D(u)|$. Since $\hat u$ minimizes $\hat D$, $\hat D(\hat u)\le\hat D(u^\star)$,
so
$
D(\hat u)\le\hat D(\hat u)+g\le\hat D(u^\star)+g\le D(u^\star)+2g,
$
which is \eqref{eq:core}. The bound is attained, so no smaller quantity governs regret; and $g$ is
a supremum over the candidate set, which by construction contains off-manifold plans, whereas the
monitored error is an average over $\mathcal M$---the two are compared next.
\end{proof}

\noindent The gap $g=\sup_u|\hat D-D|$ is the \emph{control tax}: the price the planner pays for
predicting a plan-cost that differs from the truth at the plan it most wants to believe. The three
lemmas that follow establish that the monitored prediction error is structurally unable to bound
$g$, along three independent axes---distribution, sign, and nonlinearity.

\begin{mainlem}[Decoupling: the averaged error does not bound the gap]
\label{lem:decouple}
If $\nu\ll\mu$ with density $w=d\nu/d\mu$, then $\delta_\nu^2\le\|w\|_\infty\,\delta_\mu^2$. But if
$\nu$ places any mass off $\mathrm{supp}\,\mu$, then $\delta_\mu$ gives \emph{no} upper bound on
$\delta_\nu$: there is a residual with $\delta_\mu=0$ and $\delta_\nu$---and hence the gap $g$ of
Theorem~\ref{thm:core}---arbitrarily large. Since control success is monotone in $\delta_\nu$,
every in-manifold quantity ($\delta_\mu$ and hence the monitored MSE, the content frontier, and the
spectral tax of Section~\ref{sec:pricing}) places no bound on success once actions carry the query
off support, and
\begin{equation}
    \rho\big(\mathrm{MSE},\ \mathrm{success}\big)\ \longrightarrow\ 0 .
    \label{eq:decouple}
\end{equation}
\end{mainlem}
\begin{proof}[Proof sketch]
In-support this is change of measure, $\int\!\|\varepsilon\|^2 d\nu=\int\!\|\varepsilon\|^2 w\,d\mu
\le\|w\|_\infty\delta_\mu^2$. Off-support the density is infinite and the two integrals are
independent: set $\varepsilon\equiv0$ on $\mathrm{supp}\,\mu$ (so $\delta_\mu=0$) and free
elsewhere. With $\delta_\mu$ unable to control $\delta_\nu$ and success a function of $\delta_\nu$,
no functional relation ties $\delta_\mu$ to success. Full proof:
Appendix~\ref{sec:module-H}, Theorems~\ref{thm:H1}--\ref{thm:H2}.
\end{proof}

\begin{mainlem}[The consequential part of the gap is one-sided (holes)]
\label{lem:holes}
With signed error $\Delta D=\hat D-D$, a \emph{hole} is a candidate that looks best in prediction
yet is truly bad, of depth $\mathrm{Hole}=\sup_u(-\Delta D)_+$. The suboptimality obeys
\begin{equation}
    D(\hat u)-D(u^\star)\ \le\ \Delta D(u^\star)+\mathrm{Hole},
    \label{eq:hole}
\end{equation}
and the rollout MSE---a two-sided, in-manifold average---cannot bound the hole depth, a one-sided,
off-manifold supremum. Only underestimates---assigning a low predicted cost to a truly high-cost
action---affect the suboptimality, which is why in-manifold fixes such as autoregressive rollout
training, which reduce exposure bias on the support, leave success unchanged.
\end{mainlem}
\begin{proof}[Proof sketch]
Since $\hat u$ minimizes $\hat D$, $D(\hat u)=\hat D(\hat u)-\Delta D(\hat u)\le\hat D(u^\star)-
\Delta D(\hat u)=D(u^\star)+\Delta D(u^\star)-\Delta D(\hat u)$ with $-\Delta D(\hat u)\le\mathrm{Hole}$.
A model exact on $\mu$ (MSE $=0$) can still assign a low predicted cost to a $\mu$-null
counterfactual whose true cost is high, making the hole depth unbounded; MSE and hole depth differ in
distribution, aggregation (mean vs.\ supremum), and sign. Appendix~\ref{sec:module-H},
Theorems~\ref{thm:H3}--\ref{thm:H5}.
\end{proof}

\begin{mainlem}[The gap is seeded only by realizable nonlinearity]
\label{lem:nonlinearity}
A deployed MLP predictor $g$ has a best linear fit $g=A^*z+B^*u+n$, and the gap splits orthogonally
as $\delta_\mu^2=\iota^2+\|n\|_{L^2(\mu)}^2$ into an insufficiency term $\iota$ and a
realizable-nonlinearity term $\|n\|$. If $\|n\|=0$ the predictor is globally linear, the predicted
cost is a single quadratic bowl, and no hole exists. Realizable nonlinearity is thus necessary for a
hole; we do not claim the hole depth is monotone in $\|n\|$, and Section~\ref{sec:experiments} finds
that reducing it is not the operative lever in practice.
\end{mainlem}
\begin{proof}[Proof sketch]
$m_z=\mathbb E[m_x\mid z_t]$ is an orthogonal projection of the full-state conditional mean, giving
the Pythagorean split; if $n\equiv0$ then $g$ is affine and $\|A^*z+B^*u-z_g\|^2$ has a unique
minimizer without folding. Deployment failure therefore materializes \emph{only} when the
representation genuinely carries nonlinear content. Appendix~\ref{sec:module-H},
Theorems~\ref{thm:H7}--\ref{thm:H8}.
\end{proof}

\paragraph{Geometry of the gap.}
The bound isolates a specific failure geometry. On $\mathrm{supp}\,\mu$ the predicted and true costs
agree up to the monitored residual, so the ranking induced by $\hat D$ is reliable there. A candidate
action can map the latent state outside $\mathrm{supp}\,\mu$, where $\hat D$ is an extrapolation of a
predictor fitted on $\mu$; when such an extrapolation underestimates the true cost it realizes the
configuration termed a \emph{hole} in Lemma~\ref{lem:holes}, a candidate of low predicted cost whose
true terminal state is distant from the goal. Since the planner selects
$\hat u=\arg\min_u\hat D$, an underestimate at $\hat u$ enters the suboptimality while an overestimate
does not, which is the one-sidedness of Lemma~\ref{lem:holes}. The three lemmas describe this object
along three axes: it is not detectable by a $\mu$-average because its support is disjoint from $\mu$
(Lemma~\ref{lem:decouple}), it contributes only through underestimates (Lemma~\ref{lem:holes}), and
it does not arise for a globally affine predictor, for which $\hat D$ is a single quadratic with a
unique minimizer (Lemma~\ref{lem:nonlinearity}). Reducing it therefore requires either observations
in the affected region (counterfactual data) or a penalty that raises $\hat D$ off the manifold
($\hat D+\lambda\hat r$), rather than a lower average residual on $\mathrm{supp}\,\mu$.

\paragraph{Reading.}
Theorem~\ref{thm:core} names the target---predicted cost tracking true cost at the committed
plan---and the three lemmas locate why the monitored loss misses it: the gap lives off the data
manifold (Lemma~\ref{lem:decouple}), it is one-sided in the consequential direction
(Lemma~\ref{lem:holes}), and it is nonlinear-seeded (Lemma~\ref{lem:nonlinearity}). The operative
levers are therefore to fill holes by manifold-aware pessimism ($\hat D+\lambda\hat r$) or to shrink
the gap's nonlinear component---not to lower MSE. To turn the objective into computable quantities
we now price the gap.

\section{Pricing the Predicted-versus-True Gap}
\label{sec:pricing}

Under the linear-control premise---the deployed latent dynamics are the learned linear model
\eqref{eq:dyn}---the gap $\sup_u|\hat D-D|$ separates by where it is generated. On the data manifold
the predicted and true dynamics agree up to a small residual $\delta_0$; away from it, an action
drives the state off $\mathcal M$ and the two \emph{diverge}. This off-manifold divergence, not the
on-manifold residual, is the binding term and the object of our main claim---the experiments confirm
that the on-manifold prediction error does not govern control while the off-manifold divergence does
(Section~\ref{sec:experiments}). We therefore price the off-manifold divergence first
(Section~\ref{ssec:rollout-tax}), record the on-manifold residual second
(Section~\ref{ssec:spectral-tax}), and assemble them into the control-error bound
(Section~\ref{ssec:control-error}). Both are driven by the non-normality of $P_0$: only when $P_0$ is
non-normal---it does not commute with its adjoint, the self-adjoint reversible case being the
normal prototype---do its singular values separate from its spectrum.

\subsection{The off-manifold divergence: the rollout tax}
\label{ssec:rollout-tax}
Actions couple the model to the world, and this is where the predicted and true dynamics part: an
action drives the state off $\mathcal M$, where the residual grows and compounds along the rollout.
This is the term the experiments identify as binding (Section~\ref{sec:experiments}), and the one
no data-averaged error can bound.

\begin{mainthm}[Action excitation and telescoping drift]
\label{thm:drift}
Random actions of width $\sigma_u$ push the predicted state off the manifold with energy set by an
off-manifold reachability Gramian, $\mathbb E\,r(\hat z_H)^2\le 2\rho_{\mathcal M}^2+2\sigma_u^2\,
\mathrm{tr}\,W_H^\perp$, which is zero exactly when $\mathcal M$ is invariant under $\hat A$ and
$\hat B$ is tangent. Along any trajectory the single-step residual telescopes,
\begin{equation}
    \|e_H\|\le G_H\,(\delta_0+L\,r_{\max}),\qquad G_H=\textstyle\sum_{i<H}\|\hat A^i\| ,
    \label{eq:drift}
\end{equation}
with $\delta_0$ the on-manifold residual priced in Section~\ref{ssec:spectral-tax}.
\end{mainthm}
\begin{proof}[Proof sketch]
The excursion follows from $r(\hat z_H)=\|P^\perp\hat z_H\|$, independence of the whitened actions,
and $\mathrm{tr}\,W_H^\perp=\sum_j\|P^\perp\hat A^j\hat B\|_F^2$; the drift is the discrete Duhamel
expansion of $e_{k+1}=\hat A e_k+\varepsilon_k$ (the control terms cancel), and $G_H$ is the
drift-amplification factor, $\le H$ in the near-normal regime. Appendix~\ref{sec:module-R},
Theorems~\ref{thm:D1}--\ref{thm:D3}.
\end{proof}

\paragraph{Accumulation under rollout.}
The single-step residual $\varepsilon_k$ is the quantity training reduces, whereas the planner scores
an $H$-step rollout rather than a single transition. Unrolling the error recursion
$e_{k+1}=\hat A e_k+\varepsilon_k$ (a discrete Duhamel expansion) propagates each residual by powers
of $\hat A$ and sums them, so the horizon-$H$ error is amplified by $G_H=\sum_{i<H}\|\hat A^i\|$. For
a near-normal $\hat A$ this factor satisfies $G_H\lesssim H$; under non-normality
$\|\hat A^i\|$ can be large over the transient regime and $G_H$ grows faster than the horizon. This
is why the linear state readout of Section~\ref{ssec:state-readout}, by propagating the
goal-relevant error through the near-normal $A_\varphi$ in place of $\hat A$, reduces $G_H$ to
$G_H^\varphi\le H$.

\subsection{The on-manifold residual: the spectral tax}
\label{ssec:spectral-tax}
On the data manifold the predicted and true dynamics agree up to the residual $\delta_0$, the value
the gap inherits before any action is taken. We price it here for completeness, and because it
locates the prior-work regime; it is not the binding term, and the experiments show it is driven
small by training and does not by itself separate control outcomes (the decoupling of
Section~\ref{ssec:control-exp}). Let $C_d(\varepsilon)$ be the largest content a $d$-dimensional
encoder can carry while its single-step error stays at most $\varepsilon$; the \emph{spectral tax}
is $\Pi_d(\varepsilon)=\sum_{i\le d}\sigma_i^2-C_d(\varepsilon)$.

\begin{mainthm}[Spectral pricing and its zero-cost boundary]
\label{thm:spectral}
For $d=1$, with real-section pseudospectral radius $\rho^{\mathbb R}(\delta)=\sup\{|\theta|:\theta\in
\mathbb R,\ d(\theta)\le\delta\}$,
\begin{equation}
    \Pi_1(\varepsilon)=\sigma_1^2-\rho^{\mathbb R}(\sqrt\varepsilon)^2-\varepsilon+O(\sqrt\varepsilon).
    \label{eq:pricing}
\end{equation}
Non-normality raises the leading singular value above the spectral radius,
$\sigma_1>\rho(P_0)$, which widens the tax; for a defective (Jordan) block of size $m$ it
climbs as a slow fractional power $\Pi_1(\varepsilon)=1-\Theta(\varepsilon^{1/m})$, and the
$d\ge2$ frontier obeys the staircase law $C_d(\varepsilon)-(d-1)\asymp\varepsilon^{d/m}$. When
$P_0$ is self-adjoint the tax vanishes identically, $\Pi_d(\varepsilon)=0$, recovering the
regime of \citet{balestriero2025lejepa,klindt2026does}.
\end{mainthm}
\begin{proof}[Proof sketch]
The frontier is bracketed by feasibility of a real pseudo-eigenvector at radius
$\rho^{\mathbb R}(\sqrt\varepsilon)$ and the singular-value ceiling $\sigma_1^2$; a
complex-conjugate pair sets the starting threshold $\varepsilon^*(1)=r^2$ (a purely spectral effect
that disappears at $d=2$), while non-normality both lowers that threshold and raises the ceiling.
Self-adjointness collapses $\sigma_1\to\rho$ and recovers the prior-work regime.
Appendix~\ref{sec:module-F}, Theorem~\ref{thm:F4}; Appendix~\ref{sec:module-M} for $d\ge2$.
\end{proof}

\paragraph{Role of non-normality.}
For a normal operator the singular values equal the moduli of the eigenvalues, so the content
attainable at a given predictability is determined by the spectrum and $\Pi_d(\varepsilon)=0$; this is
the self-adjoint configuration underlying the identifiability results of
\citet{balestriero2025lejepa,klindt2026does}. When $P_0$ is non-normal its eigenvectors are
non-orthogonal and its leading singular value exceeds its spectral radius, $\sigma_1>\rho(P_0)$, the
difference measuring the transient amplification that precedes asymptotic decay. The
$\varepsilon$-pseudospectrum---the set of eigenvalues of order-$\varepsilon$ perturbations of
$P_0$---extends into the complex plane by a corresponding amount, and the frontier is governed by its
real section $\rho^{\mathbb R}(\delta)$ rather than by the spectrum. The same non-normality enlarges
the rollout amplification $G_H$ of Section~\ref{ssec:rollout-tax}, so the on-manifold residual and
its off-manifold accumulation share a single source.

\subsection{The control-error theorem: assembling the two taxes}
\label{ssec:control-error}

Combining the on-manifold residual $\delta_0$ (priced by the spectral tax) and the off-manifold
excursion (priced by the rollout tax) turns the abstract gap of Theorem~\ref{thm:core} into a
computable bound.

\begin{mainthm}[Control suboptimality and the off-manifold tax]
\label{thm:subopt}
A model-predictive planner ranking candidates by predicted terminal cost incurs
\begin{equation}
    D(\hat u)-D(u^\star)\ \le\ 2\,G_H\big(\underbrace{\delta_0}_{\text{metric (on-manifold) tax}}
    +\underbrace{L\rho_{\mathcal M}+L\sigma_u\sqrt{\mathrm{tr}\,W_H^\perp}}_{\text{extrapolation (off-manifold) tax}}\big).
    \label{eq:subopt}
\end{equation}
The excursion term $\Pi_{\mathrm{off}}=2G_H L\sqrt{\mathrm{tr}\,W_H^\perp}$ is the price actions pay
for leaving the manifold; it is zero iff $L=0$ or the reachability is tangent, and it admits a
unique optimal exploration width $\sigma_u^\star$.
\end{mainthm}
\begin{proof}[Proof sketch]
Bounding the gap of Theorem~\ref{thm:core} by $\sup_u\|e_H(u)\|$ via the reverse triangle
inequality, and substituting the telescoping drift \eqref{eq:drift} and the excursion bound of
Theorem~\ref{thm:drift}, gives \eqref{eq:subopt}; balancing marginal reach against marginal
excursion gives $\sigma_u^\star$. Appendix~\ref{sec:module-R}, Theorem~\ref{thm:D4} and
Corollary~\ref{corr:D5}.
\end{proof}

\noindent Equation~\eqref{eq:subopt} is the quantitative form of the core theorem. The gap
$\sup_u|\hat D-D|$ splits into a \emph{metric tax}, the on-manifold residual $\delta_0$ that a
more Koopman-linearizable representation lowers, and an \emph{extrapolation tax}, the off-manifold
term $L\sigma_u\sqrt{\mathrm{tr}\,W_H^\perp}$ that no data-average bounds and that is reduced only by
distribution change (counterfactual data or tangent control); the amplification $G_H$ multiplies
both. The extrapolation tax is the binding term: it is the divergence between the predicted and true
dynamics off the manifold, and the experiments show it is what tracks control success while the
metric tax $\delta_0$ is driven small and decouples from it. This yields three falsifiable
predictions, tested in
Section~\ref{sec:experiments}: (i) across checkpoints of equal training loss, control success
varies and is uncorrelated with MSE once exploration pushes off support; (ii) an off-manifold
amplification signal, not the single-step error, tracks the binding constraint; and (iii) on
synthetic systems, suboptimality is zero under safe (tangent) control and grows only in proportion
to $L\sqrt{\mathrm{tr}\,W_H^\perp}$, with the bound of \eqref{eq:drift} holding throughout.

\section{A Theory-Guided Method and Experiments}
\label{sec:experiments}

The pricing of Section~\ref{sec:pricing} names the levers, and we first read off the minimal
intervention it prescribes before testing the predictions.

\subsection{A theory-guided intervention: linear state readout}
\label{ssec:state-readout}
To the latent prediction objective we add an auxiliary term that decodes the predicted latent to
the observed state through a \emph{linear} readout $W$ and penalizes its error,
\begin{equation}
    \mathcal L_{\mathrm{sr}}=\mathbb E\,\big\|W\hat z_{t+1}-s_{t+1}\big\|^2,
    \label{eq:sr}
\end{equation}
where $s_{t+1}$ is the true low-dimensional observation and $W$ is constrained to be linear; this
linearity is the operative property, as explained below. By Lemma~\ref{lem:nonlinearity} the gap
splits as
$\delta_\mu^2=\iota^2+\|n\|^2$; since $\mathcal L_{\mathrm{sr}}$ is a data-measure ($\mu$) loss it
acts on the insufficiency $\iota$, not on the off-manifold growth $L$---by the
decoupling (Lemma~\ref{lem:decouple}) no $\mu$-loss can lower $L$ pointwise. The one channel through
which an in-manifold objective still reaches off-manifold behavior is the readout's
\emph{linearity}: structural constraints extrapolate globally by algebra even though loss
\emph{values} do not cross the manifold boundary. If training pins the intertwining
$W\hat A=A_\varphi W$ on the data-spanned subspace, where $A_\varphi$ is the slow, near-normal state
dynamics, then it holds as a matrix identity everywhere, so the goal-relevant content of the rollout
error propagates by $A_\varphi$ rather than $\hat A$:
\begin{equation}
    W e_H=\sum_j A_\varphi^{\,H-1-j}\,(W\varepsilon_j),
    \qquad \|We_H\|\le G_H^\varphi\max_j\|W\varepsilon_j\|,\quad G_H^\varphi=\textstyle\sum_i\|A_\varphi^i\|\le H .
    \label{eq:intertwine}
\end{equation}
Replacing the amplification $G_H$ (which can bump under non-normality) by $G_H^\varphi\le H$ on the
goal-relevant subspace tightens the off-manifold contribution to the control suboptimality of
Theorem~\ref{thm:subopt}---by reducing the \emph{amplification} of the off-manifold residual rather
than the residual itself. The effect is therefore conditioned on non-normality (there is no
amplification to reduce when $\hat A$ is near-normal), it is orthogonal to the reachability term
$L\sigma_u\sqrt{\mathrm{tr}\,W_H^\perp}$ (which is addressed by counterfactual training or tangent
control), and it neither creates nor removes holes. A diagnostic that separates the content channel
($\iota,G_H^\varphi$) from the off-manifold residual $L$ identifies which mechanism accounts for a
measured gain.

\subsection{Control experiments on latent world models}
\label{ssec:control-exp}

The primary evidence for the theory is a set of latent-world-model control experiments; the synthetic
studies of Section~\ref{ssec:numerical-exp} corroborate the closed-form pricing separately.

\paragraph{Setup.}
We train latent world models---a ViT-tiny encoder with an autoregressive latent predictor and SIGReg
regularization---on two control tasks, \textsc{tworoom} (navigation) and \textsc{pusht}
(contact-rich pushing), and deploy them in a cross-entropy-method (CEM) model-predictive planner
(planning horizon $5$, receding horizon $5$, $300$ samples, $50$ elite trajectories). Two training
objectives are compared under a matched six-step data window across three seeds, isolating the
objective as the only difference: a single-step loss (\textsc{one\_step}) and a multi-step rollout
loss (\textsc{rollout\_full}). The monitored prediction loss is evaluated on the expert distribution
$\mu$ (teacher-forced rollouts of held-out trajectories), whereas the planner queries the model on
the reachable measure $\nu$ generated by its candidate actions, which the CEM rolls forward through
the predicted latent. During validation we additionally log an off-manifold probe: the single-step
error $\delta_1$ and a Jacobian-based amplification $\widehat L$, which measures the predictor's
sensitivity to a context perturbation leaving $\mathrm{supp}\,\mu$ and serves as a proxy for the
growth constant $L$.

\paragraph{Prediction error and control success are uncorrelated.}
The monitored loss does not order models by success. On \textsc{pusht}, the three \textsc{one\_step}
seeds attain validation prediction losses within $\pm 8\%$ ($0.0089,0.0101,0.0103$) while their
success rates span $78$--$92\%$ (Table~\ref{tab:control}). Across all runs the rank correlation
between the validation prediction loss and success is $0.024$, whereas a rollout-fidelity score
evaluated on the reachable measure $\nu$---the coefficient of determination of the predicted rollout
against the realized trajectory on the planner-reachable set---attains rank correlation $0.62$ with
success. The $\mu$-side quantity is uninformative for success and the $\nu$-side quantity is
informative, as Lemma~\ref{lem:decouple} predicts. The same separation is present within a single
training run: the checkpoint of lowest validation loss (epoch $10$, $90\%$ success) is not the
checkpoint of highest success (epoch $9$, $98\%$), so model selection by loss forgoes the
better-controlling checkpoint.

\begin{table}[h]
\centering
\small
\begin{tabular}{@{}llccc@{}}
\toprule
Task & Loss & per-seed success & mean $\pm$ std & $\delta_1$ / $\widehat L$ (ep3) \\
\midrule
\textsc{tworoom} & \textsc{one\_step}     & $90,86,88$ & $88.0\pm2.0$ & $0.012$ / $10.7$ \\
\textsc{tworoom} & \textsc{rollout\_full} & $86,90,90$ & $88.7\pm2.3$ & $0.013$ / $9.8$ \\
\textsc{pusht}   & \textsc{one\_step}     & $84,78,92$ & $84.7\pm7.0$ & $0.010$ / $11.4$ \\
\textsc{pusht}   & \textsc{rollout\_full} & $92,90,92$ & $\mathbf{91.3\pm1.2}$ & $0.010$ / $10.6$ \\
\bottomrule
\end{tabular}
\caption{Latent-MPC success (\%) on two tasks, three seeds, matched six-step training window. On
\textsc{pusht} the \textsc{one\_step} seeds share a validation prediction loss to within $\pm 8\%$
while their success spans $78$--$92\%$. Both objectives reduce the single-step error $\delta_1$ to
$\approx 0.01$, whereas the off-manifold amplification $\widehat L$ rises to $\approx 10$--$11$ and
does not decrease.}
\label{tab:control}
\end{table}

\paragraph{The off-manifold amplification is the informative diagnostic.}
Both objectives reduce the single-step error $\delta_1$ to $\approx 0.01$, while the amplification
$\widehat L$ increases to $\approx 10$--$11$ and does not decrease, since no $\mu$-supported loss
constrains the predictor off support. This is the $\delta_0\perp L$ separation of
Section~\ref{sec:control-theory} observed during training: the $\mu$-average that training minimizes
saturates, while the off-manifold quantity that governs success continues to vary. The rollout
objective's consistent structural effect is a lower multi-step compounding of $\widehat L$. On
\textsc{pusht}, where the off-manifold term is binding, this coincides with a higher and more stable
success rate ($91.3\pm1.2$ vs.\ $84.7\pm7.0$); on \textsc{tworoom}, where success is saturated, the
two objectives lie within one standard deviation. The probe thus functions as a mechanism diagnostic
rather than a per-seed predictor, and its reading is consistent with an offline decomposition in
which the latent cost is faithful near the manifold while the binding error lies in the off-manifold
dynamics.

\paragraph{Scope.}
The comparison uses three seeds, and part of the \textsc{pusht} difference is attributable to one
weak \textsc{one\_step} seed ($78\%$), so the direction is consistent but a five-seed replication is
required to establish significance. The probe reports a Jacobian proxy $\widehat L$ rather than the
off-manifold residual $L$, whose direct measurement would require simulator-supplied counterfactual
states. A model-selection rule based on the $\nu$-side fidelity score did not outperform selection by
success on these runs, which we report as a negative result. We therefore read these experiments as
establishing the decoupling and identifying the off-manifold amplification as an informative training
signal, not as a validated safety certificate.

\subsection{Cross-task validation of the state readout}
\label{ssec:sr-exp}

We evaluate the linear state readout of Section~\ref{ssec:state-readout} against the base objective on
four control tasks---\textsc{cube}, \textsc{tworoom}, \textsc{pusht}, and \textsc{reacher}---under a
common encoder, planner, and data window, adding only the auxiliary term $\mathcal L_{\mathrm{sr}}$
with weight $\lambda=3$. Control success is measured every epoch with $100$ evaluation episodes and
averaged over seeds. Figure~\ref{fig:sr} reports the per-task success trajectories and the convergence
gain of the readout over the base objective.

\begin{figure}[t]
\centering
\includegraphics[width=\textwidth]{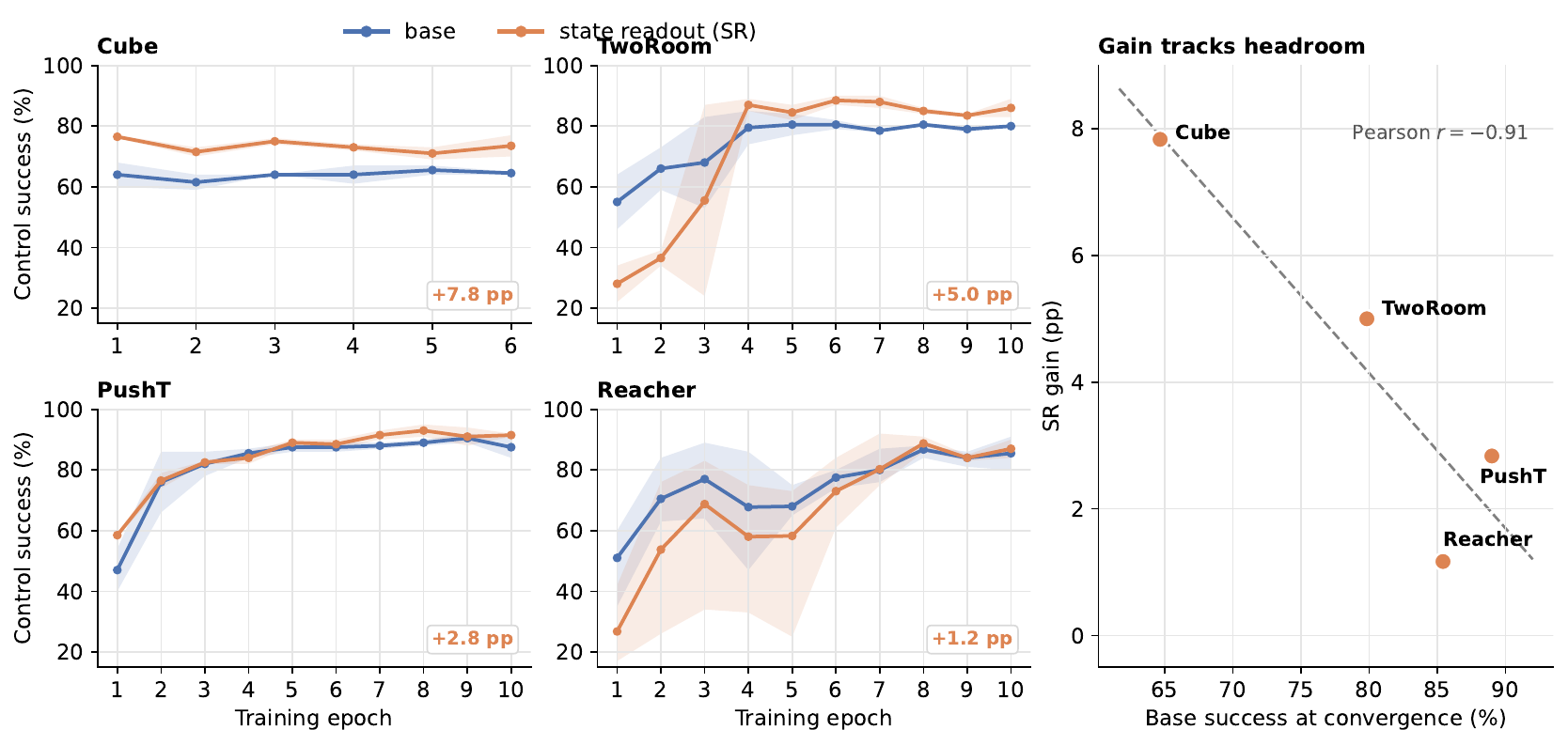}
\caption{Linear state readout (SR) versus the base objective on four control tasks. \emph{Left:}
control success against training epoch (mean over seeds; shaded band spans the seed range); the
annotation gives the convergence gain, averaged over the last three epochs. \emph{Right:} the
convergence gain against the base success at convergence, with a least-squares fit. The readout does
not lower the asymptotic ceiling, and its gain increases as the base objective moves away from
saturation (Pearson $r=-0.91$), consistent with an intervention acting on the on-manifold metric tax
$\delta_0$.}
\label{fig:sr}
\end{figure}

The readout does not lower the asymptotic ceiling on any task, and its convergence gain varies with the
base performance. The gain is largest where the base objective is least saturated (\textsc{cube},
$+7.8$~pp at a base level of $65\%$) and decreases as the base approaches saturation (\textsc{tworoom}
$+5.0$, \textsc{pusht} $+2.8$, \textsc{reacher} $+1.2$), the gain and the base convergence level
correlating at $r=-0.91$. This matches the account of Section~\ref{ssec:state-readout}: the readout is
a $\mu$-side objective acting on the on-manifold metric tax $\delta_0$, so its effect scales with the
residual on-manifold headroom and is within seed noise once that headroom is exhausted
(\textsc{reacher}). The comparison uses two to four seeds per task; a five-seed replication is
required to establish the significance of the smaller gains.

\subsection{Numerical corroboration on synthetic operators}
\label{ssec:numerical-exp}

Two synthetic studies verify the closed-form pricing of Section~\ref{sec:pricing} in isolation from
estimation error.

\paragraph{Spectral-tax frontier.}
On constructed operators the pricing formulas of Section~\ref{ssec:spectral-tax} are recovered
(Appendix~\ref{sec:numerics}). For a three-state non-normal baseline the spectral prediction of the
activation threshold is $b^2=0.0175$ and the non-normal floor is $\min(p^2,q^2)=0.00646$, in
agreement with a grid search; non-normality lowers the activation threshold and raises the completion
cost. The Jordan scaling $\Pi_1=1-\Theta(\varepsilon^{1/m})$ is recovered as slopes
$0.500,0.337,0.256$ for $m=2,3,4$ against the predicted $1/m$, and the $d\ge2$ staircase exponent
$d/m$ is recovered for $(m,d)=(4,2),(6,2),(6,3)$, with prefactor approaching $1$
(Table~\ref{tab:exponents-body}).

\begin{table}[h]
\centering
\small
\begin{tabular}{@{}lccc@{}}
\toprule
Operator $(m,d)$ & Predicted exponent & Fitted slope \\
\midrule
$J_2$ $(2,1)$ & $0.500$ & $0.500$ \\
$J_3$ $(3,1)$ & $0.333$ & $0.337$ \\
$J_4$ $(4,1)$ & $0.250$ & $0.256$ \\
$(4,2)$ & $0.500$ & $0.505$ \\
$(6,2)$ & $0.333$ & $0.346$ \\
$(6,3)$ & $0.500$ & $0.510$ \\
\bottomrule
\end{tabular}
\caption{Spectral-tax frontier exponents recovered by fitting the content frontier as a power of the
budget; the prefactor $C_d(\varepsilon)/\varepsilon^{d/m}\to1$ in every case
(Theorem~\ref{thm:spectral}).}
\label{tab:exponents-body}
\end{table}

\paragraph{Off-manifold suboptimality bound.}
Latent control systems are constructed while varying whether the data manifold is invariant under
$\hat A$ (\textsc{aligned}) or not (\textsc{off}) and the nonlinearity scale $\kappa$ that sets $L$.
Table~\ref{tab:offmanifold} reports the mean MPC suboptimality against the bound of
Eq.~\eqref{eq:drift}. Under safe control---an invariant manifold ($\mathrm{tr}\,W_H^\perp\approx0$)
or a linear plant ($L\approx0$)---the suboptimality is zero despite non-trivial reachability. When
both off-manifold reachability and nonlinearity are present it grows with $L$, from $0.61$ at
$\kappa=0.5$ to $11.9$ at $\kappa=1.5$, and the bound holds in every row. The prediction quality
$\delta_0$ is small throughout, so the suboptimality is governed by the off-manifold term
$L\sqrt{\mathrm{tr}\,W_H^\perp}$, in agreement with Theorem~\ref{thm:subopt}.

\begin{table}[h]
\centering
\small
\begin{tabular}{@{}lccccc@{}}
\toprule
Regime & $\kappa$ & $\delta_0$ & $\widehat L$ & $\mathrm{tr}\,W_H^\perp$ & suboptimality (bound) \\
\midrule
\textsc{aligned} & $0.0$ & $0$ & $0$ & $0$ & $0.00$ \ ($0$) \\
\textsc{aligned} & $0.5$ & $0.21$ & $0.50$ & $1.5\!\times\!10^{-4}$ & $0.00$ \ ($2.4$) \\
\textsc{aligned} & $1.5$ & $0.63$ & $1.49$ & $1.3\!\times\!10^{-3}$ & $0.00$ \ ($7.2$) \\
\textsc{off} & $0.0$ & $0$ & $0$ & $3.91$ & $0.00$ \ ($0$) \\
\textsc{off} & $0.5$ & $0.21$ & $0.50$ & $3.84$ & $0.61$ \ ($20.5$) \\
\textsc{off} & $1.5$ & $0.64$ & $1.51$ & $4.11$ & $11.9$ \ ($564$) \\
\bottomrule
\end{tabular}
\caption{Synthetic latent-MPC suboptimality vs.\ the bound of Eq.~\eqref{eq:drift}. Suboptimality is
zero whenever control is safe ($\mathrm{tr}\,W_H^\perp\approx0$) or the system is linear
($L\approx0$), and grows with $L$ only when both off-manifold reachability and nonlinearity are
present; the bound holds in every row.}
\label{tab:offmanifold}
\end{table}

\noindent Taken together, the two studies verify the two terms of Theorem~\ref{thm:subopt} in
isolation---the on-manifold residual priced by non-normality and the off-manifold suboptimality
priced by $L\sqrt{\mathrm{tr}\,W_H^\perp}$---while the control experiments of
Section~\ref{ssec:control-exp} exhibit the same decomposition in a trained, pixel-based system.

\section{Discussion and Outlook}
\label{sec:discussion}

\paragraph{What the theory prescribes.}
The control objective is a single scalar---the gap between predicted and true plan-cost at the
committed plan (Theorem~\ref{thm:core})---and the pricing of Theorem~\ref{thm:subopt} splits it into
a metric tax and an extrapolation tax that respond to disjoint interventions. The metric tax $\delta_0$ is
an on-manifold, $\mu$-side quantity: it is lowered by any objective that makes the latent more
Koopman-linearizable, including the linear state readout of Section~\ref{ssec:state-readout}, which
also shrinks the amplification $G_H$ on the goal-relevant subspace. The extrapolation tax
$L\sigma_u\sqrt{\mathrm{tr}\,W_H^\perp}$ is an off-manifold, $\nu$-side quantity that no
data-averaged loss bounds; it is reduced only under distribution change---counterfactual training
that shrinks $L$ by expanding $\mathrm{supp}\,\mu$, or manifold-aware pessimism
($\hat D+\lambda\hat r$) that fills holes on the measure side---or by keeping control tangent to the
manifold. Once the planner leaves the data manifold, lowering prediction error does not reduce the
extrapolation tax; the interventions that do are hole-filling and nonlinearity reduction. The two
are orthogonal, and which one binds is read from $\mathrm{tr}\,W_H^\perp$ and $L$.

\paragraph{Scope and limitations.}
The pricing of Section~\ref{sec:pricing} is stated under a linear-control premise: the deployed
dynamics are the learned linear model \eqref{eq:dyn}, and the constants in
Theorem~\ref{thm:subopt}(iii) depend on the plant geometry (for a curved $\mathcal M$ the normal
projection is local, with constants set by the second fundamental form). The control experiments use
three seeds and a Jacobian amplification proxy rather than simulator-supplied counterfactual ground
truth, so they establish the decoupling and identify the off-manifold signal but do not yet validate
it as a safety monitor. The spectral-tax and off-manifold-bound experiments are on synthetic operators
with known closed forms, isolating the mathematics of the theorems from estimation error.

\paragraph{Outlook.}
Three directions follow directly. First, establishing sharp constants for the extrapolation tax---in
particular the RKHS/second-order characterization relating $L$ to the realizable nonlinearity
$\|n\|$---would turn Theorem~\ref{thm:subopt} into a computable safety certificate. Second,
validating the off-manifold amplification $\widehat L$ as an online monitor with simulator-supplied
counterfactuals, and replicating the \textsc{pusht} decoupling at five seeds, would upgrade the
empirical claims from consistent to demonstrated. Third, the linear state readout and counterfactual
training address the two taxes separately; measuring where a combined objective lands relative to the
metric-tax/extrapolation-tax boundary is the natural test of the two-region structure.

\section{Conclusion}
\label{sec:conclusion}

We recast the objective of a latent world model deployed for planning. Its regret is bounded by a
single quantity---the gap between predicted and true plan-cost at the plan the planner commits
to---so the correct control objective is to make predicted cost track true cost, not to lower
data-averaged prediction error, which we prove neither bounds nor tracks that gap once actions carry
the query off the data manifold. Pricing the gap under a linear-control premise separates a small
on-manifold residual, where predicted and true dynamics agree, from an off-manifold divergence, where
an action carries the state off the data manifold and the two part; the divergence is the binding
term, bounded by no data-average, and the on-manifold residual---priced by a spectral tax that
vanishes on the self-adjoint boundary of prior work---does not by itself govern control. Latent-MPC
experiments establish exactly this: the decoupling of on-manifold prediction error from success, and
the off-manifold divergence as the quantity that tracks it; synthetic operators corroborate the
pricing formulas and the off-manifold suboptimality bound. The implication is that, once the planner
leaves the data manifold, suboptimality is reduced by filling holes and shrinking the
representation's nonlinear component rather than by lowering prediction error.

\newpage
\bibliographystyle{unsrtnat}
\bibliography{references}

\newpage
\appendix

\medskip
\noindent\textbf{Appendix overview.} These appendices supply the definitions and full proofs behind
the main text, in the same order. Appendix~\ref{sec:setup} fixes the setup and decomposes the
predicted-versus-true gap that governs control; Appendix~\ref{sec:preliminary} records the preliminary
lemmas and the self-adjoint reference boundary; Appendix~\ref{sec:module-H} proves the control
objective together with the decoupling, hole, and nonlinearity results (Theorem~\ref{thm:core} and
Lemmas~\ref{lem:decouple}--\ref{lem:nonlinearity}); Appendices~\ref{sec:module-F}
and~\ref{sec:module-M} price the on-manifold residual (Theorem~\ref{thm:spectral} at $d=1$ and
$d\ge2$); Appendix~\ref{sec:module-R} proves the rollout-drift and control-suboptimality bounds
(Theorems~\ref{thm:drift} and~\ref{thm:subopt}); and Appendix~\ref{sec:numerics} records the
numerical-verification protocols.

\section{Setup: the true and predicted plan-costs and their gap}
\label{sec:setup}

\paragraph{The two plan-costs and the object of the theory.}
A latent world model is deployed inside a planner. For a goal latent $z_g$ and a candidate action
sequence $u\in\mathcal U$, the planner rolls the learned latent dynamics
$z_{t+1}=\hat A z_t+\hat B u_t+\varepsilon_t$ forward and ranks $u$ by the \emph{predicted} terminal
cost $\hat D(u)=\|\hat z_H(u)-z_g\|$, executing $\hat u=\arg\min_{u}\hat D(u)$; the realized outcome is
the \emph{true} terminal cost $D(u)=\|z_H(u)-z_g\|$, minimized by $u^\star=\arg\min_u D(u)$. Control
quality is governed by the discrepancy between these two functions over the candidate set,
\begin{equation}
    g(\mathcal U)\ :=\ \sup_{u\in\mathcal U}\big|\hat D(u)-D(u)\big|,
    \label{eq:app-gap}
\end{equation}
which bounds the planner's suboptimality through $D(\hat u)-D(u^\star)\le 2\,g(\mathcal U)$
(Theorem~\ref{thm:core}). The appendix is organized around bounding $g(\mathcal U)$.

\paragraph{Two measures.}
The model is trained on the data measure $\mu$ but, at deployment, is queried on the states its
candidate actions reach---the reachable measure $\nu$, generically with
$\mathrm{supp}\,\nu\supsetneq\mathrm{supp}\,\mu$. Reading a single residual function $\varepsilon$
under the two measures yields the monitored training error
$\delta_\mu^2=\mathbb E_\mu\|\varepsilon\|^2$ and the deployment error
$\delta_\nu^2=\mathbb E_\nu\|\varepsilon\|^2$; their separation is the source of the decoupling
(Appendix~\ref{sec:module-H}).

\paragraph{Decomposition of the gap.}
Along a rolled trajectory the terminal error $e_H=\hat z_H-z_H$ obeys the recursion
$e_{k+1}=\hat A e_k+\varepsilon_k$, so $g(\mathcal U)$ is controlled by the single-step residual
$\varepsilon$ amplified over the horizon (Appendix~\ref{sec:module-R}). The residual splits into an
\emph{on-manifold} part---its value $\delta_0$ on $\mathrm{supp}\,\mu$, determined by how well a
$d$-dimensional latent can be linearly predicted and priced in
Appendices~\ref{sec:module-F}--\ref{sec:module-M}---and an \emph{off-manifold} part governed by the
distance $r(z)=\mathrm{dist}(z,\mathcal M)$ to the data manifold $\mathcal M=\mathrm{supp}\,\mu$,
\begin{equation}
    \|\varepsilon(z,u)\|\ \le\ \delta_0+L\,r(z),
    \label{eq:app-residual}
\end{equation}
with $L$ a nonlinearity scale ($L=0$ for a globally linear system). Equations~\eqref{eq:app-gap}
and~\eqref{eq:app-residual} are the two decompositions that the remaining appendices make
quantitative.

\paragraph{Process and transition operator.}
Let $\{X_t\}$ be a stationary ergodic process satisfying a first-order Markov assumption
\textbf{(M)}. Non-Markovian streams such as raw video are brought into this setting by state
augmentation. Let $P$ be the transition (Koopman) operator on $L^2(\mu)$. Restricting to the
mean-zero subspace $H_0 = \mathbf{1}^\perp \subset L^2(\mu)$ removes the constant direction;
we write $P_0 = P|_{H_0}$ for the restricted operator, which is real and satisfies
$\|P_0\| \le 1$. The Hilbert--Schmidt (compactness) assumption is used only where stated. A
whitened encoder $\phi$ corresponds to an isometry $V:\mathbb{R}^d \to H_0$; under whitening
the optimal predictor bias vanishes, $b \equiv 0$ (Lemma~\ref{lem:B1}).

\paragraph{On-manifold pricing objects.}
The on-manifold residual $\delta_0$ is analyzed through the transition operator and the encoder. The
\emph{content} of an encoder is $\mathcal C(\phi)=\|P_0V\|_{\mathrm{HS}}^2$ and its single-step
\emph{gap} is $\delta(\phi)^2=\|(I-VV^*)P_0V\|_{\mathrm{HS}}^2$, with optimal within-class predictor
$A^*=V^*P_0V$. The largest content attainable at gap at most $\varepsilon$ is the \emph{content
frontier} $C_d(\varepsilon)=\sup\{\mathcal C(\phi):\delta(\phi)^2\le\varepsilon\}$; its shortfall from
the unconstrained optimum is the \emph{spectral tax}
$\Pi_d(\varepsilon)=\sum_{i\le d}\sigma_i^2-C_d(\varepsilon)$, with activation threshold
$\varepsilon^*(d)=\inf_\phi\delta(\phi)^2$. Two further constants enter the off-manifold analysis: the
memory horizon $\tau=\sum_j\|P_0^j\|$, the cumulative non-normal transient, and a regularity budget
$\mathcal B=\log(L\Lambda)$ with domain-local version $\mathcal B_{\mathrm{dom}}\le\mathcal B$. These
objects price $\delta_0$ and its amplification; they are the machinery of the bound, not the control
objective itself.

\paragraph{Three-term decomposition of the single-step residual.}
The single-step prediction loss decomposes orthogonally as
\begin{equation}
    \text{single-step loss}
    = \underbrace{\text{(I) innovation}}_{\text{conditioned on } X_t}
    + \underbrace{\text{(II) sufficiency gap}}_{I(Z';X) - I(Z') \text{ deficit}}
    + \underbrace{\text{(III) realizability error}}_{\text{within-class predictor approximation}} ,
\end{equation}
with systematic predictable part $G^* = \text{(II)} + \text{(III)}$. Appendix~\ref{sec:module-H}
(Theorem~\ref{thm:H7}) identifies the realizability term (III) as the realizable nonlinearity that
seeds holes; term (II) is the encoder's insufficiency.

\paragraph{Normal and non-normal transitions.}
We call $P_0$ normal when it commutes with its adjoint, equivalently when it is unitarily
diagonalizable; the self-adjoint case $P_0 = P_0^*$ is the prototype, realized by reversible dynamics
such as the Ornstein--Uhlenbeck process. We call $P_0$ non-normal otherwise. The two cases that drive
the on-manifold results are defective (non-diagonalizable) operators and finite-memory nilpotent
operators; for these the singular values separate from the spectrum, and this separation prices the
on-manifold residual.

\paragraph{The self-adjoint reference boundary.}
The objects above are calibrated against one reference point. When $P_0$ is self-adjoint the spectral
tax vanishes, $\Pi_d(\varepsilon) = 0$, and the construction reduces to the regime analyzed by
LeJEPA \citep{balestriero2025lejepa,klindt2026does}: the loss admits a closed form through the
Ky Fan maximum principle, and the linear identifiability of that work is recovered
(Corollary~\ref{corr:P2.1}, which matches their Theorem~1). What matters here is the location of this
optimum: it places the LeJEPA guarantees on the normal, self-adjoint boundary of the frontier, the
reference against which every positive cost in the sequel is measured.

\section{Preliminaries and the self-adjoint boundary}
\label{sec:preliminary}

The following five preliminary theorems provide complete statements alongside their inline proofs.

\paragraph{P1 (Saturation $\iff$ Invariance; The Gap Identity).}
\label{thm:P1}
Let $\phi$ be a whitened encoder corresponding to the isometry $V$, and let $S = \mathrm{ran}\,V$. The single-step gap within the linear predictor class satisfies $G^*(\phi) = \delta(S)^2 = \|(I - \Pi_S)P_0\Pi_S\|_F^2$. In particular, $G^* = 0 \iff \mathrm{span}\{\mathbf{1}, \phi\}$ is a $P$-invariant subspace.

\begin{proof}
Under the whitening constraint, the Gram matrix is exactly the identity $I$, and the normal equations yield the optimal within-class coefficient $A^* = V^\top P_0 V$. The optimal prediction residual evaluates to $\|P_0V - VA^*\|_F^2 = \|(I - \Pi_S)P_0V\|_F^2 = \delta^2$. Setting this residual to zero is equivalent to requiring $P_0S \subseteq S$; incorporating the constant direction $\mathbf{1}$ yields the stated subspace invariance.
\end{proof}
\begin{remark}
Folded variants of this invariance equivalence can be found in \citet{brunton2016koopman}, \citet{otto2019linearly}, and \citet{lusch2018deep}. The contribution of this work lies not in the equivalence itself, but in establishing the quantitative cost frontiers once the system departs from the invariant subspace, formalized in modules F/M.
\end{remark}

\paragraph{P2 (The Self-Adjoint Solvable Regime).}
\label{thm:P2}
Assume $P_0$ is self-adjoint (corresponding to the VAC/SFA setting; see \citep{noe2013variational,wiskott2002slow}). The top-$d$ eigenspace is $P_0$-invariant, which by Theorem~\ref{thm:P1} yields a zero prediction gap ($G^* = 0$) and retains a total content of $\sum_{i \le d} \lambda_i^2$, which is strictly optimal within the class of zero-gap encoders. 

Two critical corrections are applied: a strict spectral gap $|\lambda_1| < 1$ requires the dynamics to be aperiodic, and the extraction of eigenvectors at degenerate eigenvalues must be performed via a strict orthogonal selection.

\begin{proof}
The spectral theorem combined with Theorem~\ref{thm:P1} guarantees both subspace invariance and the exact content formulation. For optimality, any invariant subspace corresponds to a spectral subspace with content $\sum \lambda_i(S)^2$. The Poincar{\'e} separation theorem (a direct corollary of Cauchy interlacing) guarantees that the top eigenspace maximizes this sum.
\end{proof}

\paragraph{Corollary P2.1 (LeJEPA's Gaussian/OU World as the Self-Adjoint, Zero-Tax Boundary; Recovery of Their Theorem 1).}
\label{corr:P2.1}
Consider the Gaussian/Ornstein-Uhlenbeck (OU) setup driving the linear identifiability guarantees of LeJEPA \citep{klindt2026does}: let $z \sim \mathcal{N}(0, I_n)$ and $z' = \rho z + \sqrt{1-\rho^2}\,\eta$, where $\eta \sim \mathcal{N}(0, I_n)$ and $\rho \in (0, 1)$. The transition operator $(T\psi)(z) = \mathbb{E}[\psi(z') \mid z]$ is self-adjoint with respect to the stationary, reversible OU measure $\mu = \mathcal{N}(0, I_n)$. Thus, $P_0 = T|_{H_0}$ is self-adjoint with $\|P_0\| = \rho$. Its eigenfunctions are the classical Hermite polynomials, and the degree-$d$ eigenvalues scale as $\rho^d$ (via Mehler's formula). 

Consequently, the top eigenspace corresponding to the $n$-fold degenerate eigenvalue $\rho$ is spanned exactly by the linear functions $\mathrm{span}\{z_1, \dots, z_n\}$. Under a whitened representation ($\mathrm{Cov}(h) = I_n$, meaning the isometry satisfies $V: \mathbb{R}^n \to H_0$ with $Ve_i = h_i$), LeJEPA's alignment loss expands as $L(h) = 2n - 2\sum_i \mathbb{E}[h_i(z')h_i(z)] = 2n - 2\,\mathrm{tr}(V^*P_0V)$. Applying Theorem~\ref{thm:P2} for $d=n$, we obtain:
\begin{itemize}
    \item[\textbf{(i)}] The top eigenspace is $P_0$-invariant, yielding a zero prediction gap $\delta = 0$ ($G^* = 0$) and capturing a maximum content of $C_n(0) = n\rho^2 = \sum_{i \le n} \sigma_i(P_0)^2$. Consequently, \textbf{the spectral tax vanishes identically}: $\Pi_n(\varepsilon) \equiv 0$ (aligning with module Theorem~\ref{thm:M2}, where self-adjointness implies $\Psi_n = 0$, and Corollary~\ref{corr:F1.1}, where the activation threshold satisfies $\varepsilon^*(d) = 0$).
    \item[\textbf{(ii)}] The Ky Fan / Poincar{\'e} separation theorem implies $\max_V \mathrm{tr}(V^*P_0V) = n\rho$, where equality holds if and only if $\mathrm{ran}\,V = \mathrm{span}\{z_i\}$, which forces the linear representation $h(z) = Qz$ for some $Q \in \mathcal{O}(n)$, achieving the minimum alignment loss $\min L = 2(1-\rho)n$. \textit{This exactly recovers Theorem~1 of \citet{klindt2026does}.}
\end{itemize}

\begin{proof}
To see that $T$ is self-adjoint, note that the stationary OU process is time-reversible, inducing a symmetric joint distribution over $(z, z')$ with $\mathrm{Cov}(z, z') = \rho I$; hence $\langle \psi, T\chi \rangle = \langle T\psi, \chi \rangle$. The identity $\langle h_i, P_0h_i \rangle = (V^*P_0V)_{ii}$ yields the trace form of the loss $L$. Property (i) follows from the invariance of the top eigenspace and the self-adjoint identity $\sigma_i(P_0) = |\lambda_i| = \rho$. Property (ii) is a direct consequence of the Poincar{\'e} separation theorem, where the trace maximization for a self-adjoint operator yields the sum of the top $n$ eigenvalues $n\rho$, achieved uniquely at the spectral subspace.
\end{proof}

\paragraph{Interpretation.} LeJEPA's linear identifiability constitutes the exact \textbf{self-adjoint, zero-tax boundary} of our cost theory. The activation threshold $\varepsilon^*$, the compression deficit $\Psi_d$, the kernel tax, the Jordan chains, and the multi-step rollout drift are all identically zero at this boundary, activating only when $P_0$ departs from normality---that is, when moving outside the Gaussian/OU framework required for their linear identifiability (their identifiability holds if-and-only-if Gaussian; see Theorems~1 and~2 in \citet{klindt2026does}). This work explicitly maps the cost surface unfolding \textit{outside} this boundary. Technically, Corollary~\ref{corr:P2.1} is an immediate consequence of the Ky Fan variational principle and Poincar{\'e} separation; its value lies entirely in its \textbf{positioning}---establishing a mathematically precise seam where our cost surface intersects and recovers LeJEPA's optimal point---rather than in its proof complexity.

\paragraph{P3 (Pathology of Deterministic Weak Mixing: Ill-Posedness vs. Impossibility).}
\label{thm:P3}
Let $T$ be a weakly mixing transformation and let $U$ denote its associated Koopman operator.
\begin{itemize}
    \item[\textbf{(a)}] There exists no non-trivial finite-dimensional invariant subspace on $\mathbf{1}^\perp$; hence, exact saturation is impossible.
    \item[\textbf{(b)}] In the time-reversible (invertible) case, the single-step gap satisfies $\inf_{\dim S = d} \delta(S)^2 = 0$ for all dimensions $d \ge 1$, including odd dimensions.
\end{itemize}

\begin{proof}
(a) A transformation is weakly mixing if and only if $U|_{\mathbf{1}^\perp}$ possesses no eigenvalues. Any finite-dimensional invariant subspace must contain at least one eigenvector upon complexification, yielding a contradiction. 

(b) Since $U$ is unitary, its spectral measure $E$ is purely continuous (atomless). 

\textit{Even Dimensions:} We choose $d/2$ pairs of disjoint, conjugate arc pairs $\{I_j, \bar{I}_j\}$ on the unit circle. The continuity of the spectral measure guarantees that these arcs can be selected to capture non-zero spectral mass. Let $f_j = E(I_j)g_j$ be normalized vectors. By forming real combinations using $\{\mathrm{Re}, \mathrm{Im}\}$ pairings, we construct an orthogonal set of size $d$ satisfying $\delta(S) \le \max_j |I_j| \to 0$. 

\textit{Odd Dimensions:} We require an additional one-dimensional real direction, which must be supported on a symmetric arc $I_0 = \bar{I}_0$ containing the point $\theta = 1$ (or $-1$). Its feasibility is guaranteed because a weakly mixing system is necessarily aperiodic, and \textbf{the maximal spectral type of any aperiodic measure-preserving system has full support over the entire unit circle} (constructed via approximate eigenfunctions on Rokhlin towers; citation level: \citep{nadkarni1998spectral}). Consequently, any arbitrary neighborhood of $1$ contains non-zero spectral mass. Choosing a real $g_0$ ensures that $f_0 = E(I_0)g_0$ is strictly real (since $I_0$ is symmetric and $U$ is a real operator, making $E(I_0)$ commute with complex conjugation) and orthogonal to all prior conjugate pairs, satisfying $\|(U - 1)f_0\| \le |I_0|$.
\end{proof}
\begin{remark}
For the continuous spectrum landscape, see \citet{mezic2005spectral}; the entropy/Pinsker structure that rules out zero-entropy factor evasion is standard ergodic theory \citep{glasner2003ergodic}. \textbf{Contrast with Theorem~\ref{thm:M1.2}:} The topological barrier $\varepsilon^*(d) > 0$ for a finite-dimensional complex spectrum in odd dimensions is completely dissolved under an infinite-dimensional weakly mixing structure because the spectral support is forced to contain $\pm 1$. The vanishing of the odd-dimensional threshold is a pathology unique to infinite dimensions.
\end{remark}

\paragraph{P4 (Decoupled vs. Shared Encoders; Quantifying the Trilemma).}
\label{thm:P4}
In a decoupled (dual-encoder) setup, the optimal single-step content equals $\sum_{i \le d} \sigma_i(P_0)^2$, which represents the universal VAMP-2 upper bound \citep{wu2020variational,mardt2018vampnets}. Under the shared encoder constraint $\psi = \phi$, the optimal content is bounded by the $C_d(\varepsilon)$ frontier. The frontier is thus the quantitative form of the content--predictability--dimensionality trade-off under the shared-encoder constraint.

\begin{proof}
The decoupled optimization objective expands as $\max_{U, V \text{ isometric}} \|U^\top P_0 V\|_F^2$. The singular value decomposition (Schmidt/Eckart--Young theorem) proves that this maximum is achieved by the sum of the top $d$ singular triplets. Restricting the system to a shared encoder $U = V$ recovers the shared-encoder configuration of Appendix~\ref{sec:setup}.
\end{proof}

\paragraph{P5 (The Single-Sided Certificate; The Population Identity).}
\label{thm:P5}
We introduce the dual-predictor certificate formulated on the raw input side:
\begin{align}
    \hat{\delta}^2_{\mathrm{cert}} &:= \min_{g \in L^2} \mathbb{E}\big\|\phi(X_{t+1}) - g(X_t)\big\|^2 - \min_{h \in L^2} \mathbb{E}\big\|\phi(X_{t+1}) - h(\phi(X_t))\big\|^2 \nonumber \\
    &= \mathbb{E}\big\|\mathbb{E}[\phi(X_{t+1}) \mid X_t] - \mathbb{E}[\phi(X_{t+1}) \mid \phi(X_t)]\big\|^2
\end{align}
This identity isolates the sufficiency gap (II) mapped onto the $\phi$-target space. It is strictly non-negative and vanishes if and only if $\phi(X_t)$ is a sufficient statistic for predicting $\phi(X_{t+1})$.

\begin{proof}
The optimal predictors within both $L^2$ spaces correspond to their respective conditional expectations by the tower property. The subspace inclusion $\sigma(\phi(X_t)) \subseteq \sigma(X_t)$ combined with the $L^2$ Pythagorean theorem yields the squared norm of the difference.
\end{proof}

\section{The predicted-versus-true gap: decoupling, holes, and realizable nonlinearity}
\label{sec:module-H}

This section proves the control objective of Section~\ref{sec:control-theory} and the three results
that follow from it. A planner (MPC or CEM in latent space; \citealp{garcia2013model,hafner2020dream})
acts on the manifold it \emph{predicts} it can reach, and its suboptimality is charged by the gap
$g(\mathcal U)=\sup_u|\hat D(u)-D(u)|$ between the predicted and true plan-cost at the plans it
considers (Theorem~\ref{thm:core}). We show that the data-averaged prediction error---an average over
the stationary measure $\mu$---does not bound this gap once candidate actions carry the query to the
reachable measure $\nu$ (the decoupling); that the part of the gap that affects suboptimality is
one-sided (holes); and that holes arise only when the predictor carries realizable nonlinearity, by
the linear--MLP bridge of Section~\ref{sec:module-D-bridge}. The on-manifold residual $\delta_0$ that
these results take as an input is priced separately in Appendices~\ref{sec:module-F}
and~\ref{sec:module-M}.

\subsection{Setup: query measure, deployment gap, and holes}

\paragraph{Query (deployment) measure.} A planner ranks candidate action sequences by rolling them in latent space and executing the one of lowest predicted cost. The queried states follow the reachable measure $\nu$, and generically $\mathrm{supp}\,\nu\supsetneq\mathrm{supp}\,\mu$: actions are the mechanism that generates the query measure $\nu$ from the data measure $\mu$.

\paragraph{The two faces of the gap.} The same residual function $\varepsilon$ read under two measures gives $\delta_\mu^2:=\mathbb E_\mu\|\varepsilon\|^2$ (the training gap $\delta$, monitored as MSE) and $\delta_\nu^2:=\mathbb E_\nu\|\varepsilon\|^2$ (the deployment gap on which success actually rests). The residual is unchanged; only the averaging measure differs.

\paragraph{Predicted cost, true cost, and holes.} With predicted cost $\hat J(u)=\|\hat z_H(u)-z_g\|^2$, true cost $J(u)=\|z_H(u)-z_g\|^2$, and signed error $\Delta J(u):=\hat J(u)-J(u)$, a \emph{hole} is a candidate $u_0$ that looks best in prediction, $\hat J(u_0)<\hat J(u^\star)$, yet is truly bad, $J(u_0)\gg J(u^\star)$: a deceptive low valley in the predicted cost landscape. The hole depth is $\mathrm{Hole}(\mathcal U):=\sup_u(-\Delta J(u))_+$.

\subsection{Measure mismatch and the decoupling of prediction error from success}

\subsubsection{Theorem H1 (Bounded Density vs.\ Off-Support Blow-Up)}
\label{thm:H1}
\begin{align}
    \text{(a) in-support } (\nu\ll\mu,\ w=d\nu/d\mu):\quad & \delta_\nu^2=\mathbb E_\mu\big[w\|\varepsilon\|^2\big]\le\|w\|_\infty\,\delta_\mu^2 ; \\
    \text{(b) off-support } (\nu(\mathbb R^d\setminus\mathrm{supp}\,\mu)>0):\quad & \|w\|_\infty=\infty,\ \text{and }\delta_\mu\text{ gives no upper bound on }\delta_\nu .
\end{align}
In case (b) there exists a residual with $\delta_\mu=0$ yet $\delta_\nu$ arbitrarily large.
\begin{proof}
(a) is the change-of-measure identity $\int f\,d\nu=\int fw\,d\mu$ with $w\le\|w\|_\infty$. (b): off $\mathrm{supp}\,\mu$ the density $w$ is infinite; let $\varepsilon\equiv0$ on $\mathrm{supp}\,\mu$ (so $\delta_\mu=0$) but take arbitrary values on a $\nu$-positive, $\mu$-null set---the two integrals are independent.
\end{proof}

\subsubsection{Theorem H2 (Decoupling: Training Error $\perp$ Control Success)}
\label{thm:H2}
By Theorem~\ref{thm:H1}(b), every $\mu$-quantity ($\delta_\mu$, the frontier $C_d$, the spectral tax $\Pi_d$) places zero constraint on $\delta_\nu$ once $\nu$ carries off-support mass. Since control success is monotone in $\delta_\nu$, the correlation obeys $\rho(\delta_\mu,\ \text{success})\to0$.
\begin{proof}
Immediate from Theorem~\ref{thm:H1}(b): if $\delta_\mu$ does not even control $\delta_\nu$, and success is determined by $\delta_\nu$, then no functional relation ties $\delta_\mu$ to success. No off-manifold-distance or ranking machinery is needed---decoupling is a direct corollary of change of measure.
\end{proof}

\begin{interpretation}
The measured MSE is a $\mu$-quantity, whereas success rests on a $\nu$-quantity, with a measure mismatch between them. This states, inside the cost-theory language, the empirical observation that single-step prediction error and downstream control success are nearly uncorrelated.
\end{interpretation}

\subsection{Holes as false minima of the predicted cost}

\subsubsection{Theorem H3 (Suboptimality as Signed Hole Depth)}
\label{thm:H3}
\begin{equation}
    J(\hat u)-J(u^\star)\ \le\ \underbrace{\Delta J(u^\star)}_{\text{overestimate of the true optimum}}+\underbrace{\mathrm{Hole}(\mathcal U)}_{\text{depth at the selected action}} .
\end{equation}
\begin{proof}
$J(\hat u)=\hat J(\hat u)-\Delta J(\hat u)\le\hat J(u^\star)-\Delta J(\hat u)=J(u^\star)+\Delta J(u^\star)-\Delta J(\hat u)$, using that $\hat u$ minimizes $\hat J$; and $-\Delta J(\hat u)$ is the underestimate at the selected action, bounded by the hole depth.
\end{proof}

\begin{interpretation}
This refines Theorem~\ref{thm:D4}(i) from the two-sided $\|e_H\|$ to a one-sided quantity: only underestimates---calling a bad action good---are fatal, while overestimates merely make the planner conservative. MSE, a two-sided square-mean, dilutes exactly this one-sided worst case; this is a second layer of the decoupling.
\end{interpretation}

\subsubsection{Theorem H4 (MSE Does Not Control Hole Depth)}
\label{thm:H4}
The rollout MSE $\mathbb E_\mu\|e_H\|^2$ (an in-manifold, two-sided mean) provides no upper bound on the hole depth $\sup_\nu(-\Delta J)_+$ (an off-manifold, one-sided worst case).
\begin{proof}
Isomorphic to Theorem~\ref{thm:H1}(b): a model exact on $\mu$ (MSE $=0$) can hallucinate arrival at the goal on a $\mu$-null counterfactual candidate ($\hat J=0$) with arbitrarily large true cost, giving an arbitrarily deep hole at zero MSE. The three decoupling sources are distribution (in vs.\ out of support), aggregation (mean vs.\ worst case), and sign (two-sided vs.\ one-sided).
\end{proof}

\subsubsection{Theorem H5 (Exposure Bias Is Orthogonal to Holes)}
\label{thm:H5}
Relative to teacher forcing, autoregressive rollout training reduces only the in-manifold exposure-bias error $e^\parallel$; it leaves the off-manifold hole error $e^\perp$ untouched. Consequently the rollout MSE decreases while control success is unchanged.
\begin{proof}
Exposure bias is the self-fed versus ground-truth-fed shift on in-support trajectories \citep{bengio2015scheduled,venkatraman2015improving}; autoregressive training aligns the two, reducing $e^\parallel$. The off-support states that seed holes are absent from both the teacher-forcing and the autoregressive training distributions, so $e^\perp$ and the hole depth do not move; by Theorem~\ref{thm:H3} success is governed by the hole depth and is therefore unchanged.
\end{proof}

\begin{interpretation}
The teacher-forcing $\to$ autoregressive $\to$ diffusion-forcing line all lives inside the manifold, fixing generation quality (exposure bias); it is orthogonal to the off-manifold holes that cause planning to fail. Multi-step rollout training thus optimizes the right metric but the wrong objective.
\end{interpretation}

\subsubsection{Proposition H6 (Pessimism Fills Holes)}
\label{prop:H6}
Let an off-manifold indicator $h(u)\ge0$ satisfy $\lambda_0 h(u)\ge(J-\hat J)_+$ on the reachable set (e.g.\ $h\propto\hat r(u)$, since the hole depth is $\lesssim$ the Theorem~\ref{thm:D4} bound, of order $Lr$). Then the pessimistic cost $\hat J_{\mathrm{pess}}=\hat J+\lambda h$ with $\lambda\ge\lambda_0$ satisfies $\hat J_{\mathrm{pess}}\ge J$, so no false minimum lies below the true optimum. An optimal $\lambda^\star$ trades hole-filling against over-conservatism, dual to $\sigma_u^\star$ of Corollary~\ref{corr:D5}.
\begin{interpretation}
The prescription is to \emph{fill holes}, not to lower MSE: either a training-free manifold-aware / pessimistic planner ($\hat J+\lambda\hat r$), or counterfactual training that shrinks $L$ by expanding $\mathrm{supp}\,\mu$.
\end{interpretation}

\subsection{The linear--MLP bridge: the origin of holes}
\label{sec:module-D-bridge}

The framework is stated for the linear Koopman operator $A^*$, while deployed predictors are MLPs. The next two results show the two are the same object: the MLP's ``extra nonlinearity'' is exactly the realizability term (III) of the three-term decomposition of Section~\ref{sec:setup}, and holes grow from it.

\subsubsection{Theorem H7 (Sufficiency Decomposition of the Gap)}
\label{thm:H7}
An MLP $g$ learns $m_z:=\mathbb E[z_{t+1}\mid z_t]$; its $\mu$-best linear fit writes $g=A^*z+B^*u+n$ with $n$ the MLP nonlinearity. With the full-state conditional mean $m_x:=\mathbb E[z_{t+1}\mid x_t]$ (the action of $P_0$),
\begin{equation}
    \delta_\mu^2=\underbrace{\mathbb E\|m_x-m_z\|^2}_{\iota^2:\ \text{insufficiency}=\text{(II)}}+\underbrace{\|n\|_{L^2(\mu)}^2}_{\text{realizable nonlinearity}=\text{(III)}}, \qquad \|n\|_\mu\le\delta_\mu ,
\end{equation}
with equality iff the latent is sufficient ($m_x$ is a function of $z_t$, $\iota=0$).
\begin{proof}
The original gap is $\delta_\mu^2=\min_A\mathbb E\|m_x-Az_t\|^2$. Since $m_z=\mathbb E[m_x\mid z_t]$ is the projection of $m_x$ onto $\sigma(z_t)$, for any linear $Az_t$ one has $m_x-m_z\perp\sigma(z_t)$, so by Pythagoras $\mathbb E\|m_x-Az_t\|^2=\iota^2+\mathbb E\|m_z-Az_t\|^2$; minimizing over $A$ sends the second term to $\|n\|_\mu^2$. This is precisely the population certificate P5 and the three-term split $\delta_\mu^2=\text{(II)}+\text{(III)}$ of Section~\ref{sec:setup}.
\end{proof}

\begin{interpretation}
The ``linear theory versus MLP practice'' gap is the framework's own object. Insufficiency $\iota$ (term II) is information the encoder discarded---no $z_t$-predictor, linear or MLP, can recover it; it is a representation problem. Realizable nonlinearity $\|n\|_\mu$ (term III) is nonlinearity the MLP carries within the predictable range, and it is what seeds holes. The identity $\|n\|_\mu=\delta_\mu$ holds only when the latent is (near-)sufficient.
\end{interpretation}

\subsubsection{Theorem H8 (Holes Grow from Realizable Nonlinearity)}
\label{thm:H8}
Let the MLP fit the data exactly, so $\varepsilon=\Phi-g$ vanishes on $\mu$. Then
\begin{equation}
    \|n\|_\mu=0\ \Rightarrow\ n\equiv0\ \Rightarrow\ g\text{ globally linear}\ \Rightarrow\ \text{a single quadratic bowl, with no false minimum.}
\end{equation}
\begin{proof}
By Theorem~\ref{thm:H7}, $\|n\|_\mu=0$ forces $n=0$, so $g$ is affine and $\|A^*z+B^*u-z_g\|^2$ has a unique minimum with no folding.
\end{proof}

\begin{interpretation}
Realizable nonlinearity $\|n\|_\mu$ is \emph{necessary} for a hole: a globally linear predictor folds no space and has none, whereas insufficiency $\iota$ merely raises the noise floor. We do not assert the converse quantitative direction, that a larger $\|n\|_\mu$ produces a deeper hole; the control experiments of Section~\ref{sec:experiments} find that an off-manifold nonlinearity proxy does not decrease with control success, so lowering $\|n\|_\mu$ is not the operative lever in practice. What the readout gain tracks empirically is target fidelity, not the reduction of $\|n\|_\mu$.
\end{interpretation}

\subsection{One closure defect, several faces}

\begin{corollary}[Zero-Tax Boundary: All Faces Vanish Together]
\label{corr:H9}
On the self-adjoint / OU boundary of Corollary~\ref{corr:P2.1} (the LeJEPA world), $\delta_\mu=0$ forces $\iota=0$ and $\|n\|_\mu=0$, so the MLP degenerates to linear and the spectral tax, the hole tax, and the gap vanish simultaneously.
\end{corollary}

\begin{table}[h]
\centering
\small
\begin{tabular}{@{}llll@{}}
\toprule
\textbf{Face} & \textbf{Object} & \textbf{Locus} & \textbf{Set by} \\
\midrule
Spectral tax $\Pi_d$ & $A^*$ non-normal & in-manifold, spectral geometry & $\sigma_1(A^*)\gg\rho(A^*)$ \\
Gap $\delta_\mu$ & linear closure residual, split $\iota^2+\|n\|_\mu^2$ & in-manifold, $\mu$-mean & $\|(I-\Pi)P_0V\|$ \\
\quad$\hookrightarrow$ insufficiency $\iota$ & encoder information loss & in-manifold, floor & $\|m_x-m_z\|$ \\
\quad$\hookrightarrow$ nonlinearity $\|n\|_\mu$ & MLP nonlinearity (necessary for holes) & in $\to$ off-manifold & $\sqrt{\delta_\mu^2-\iota^2}$ \\
\bottomrule
\end{tabular}
\caption{One Koopman closure defect---a finite latent coordinate is not a global Koopman-invariant coordinate---shows several faces, all tied together by $\delta_\mu$ and $A^*$. All taxes vanish jointly on the self-adjoint boundary (Corollary~\ref{corr:H9}).}
\label{tab:faces}
\end{table}

\paragraph{Rigor status.} Theorems~\ref{thm:H1}, \ref{thm:H2}, \ref{thm:H3}, \ref{thm:H4}, \ref{thm:H7} and Theorem~\ref{thm:H8} are exact; Theorem~\ref{thm:H5} and Proposition~\ref{prop:H6} are mechanistically exact with constants that depend on the plant geometry. We make no quantitative claim relating the off-manifold growth $L$ to the realizable nonlinearity $\|n\|_\mu$: the control experiments of Section~\ref{sec:experiments} find that an off-manifold nonlinearity proxy tracks control success in the opposite direction, so such a bound would be empirically unsupported.

\section{Pricing the on-manifold residual: the $d=1$ spectral frontier}
\label{sec:module-F}

This section prices the on-manifold residual at $d=1$ through a real-section pseudospectral
characterization of the content--predictability frontier. The pricing formula
(Theorem~\ref{thm:F4}) is the $d=1$ anchor extended to $d\ge2$ in Appendix~\ref{sec:module-M} and
taken as an input by the deployment analysis of Appendix~\ref{sec:module-R} and the gap analysis of
Appendix~\ref{sec:module-H}. The underlying machinery draws on classical Hardy- and Clark-space
operator theory \citep{garnett2006bounded,garcia2016introduction,clark1972one,sarason2007algebraic}
and on sharp resolvent and pseudospectral bounds for Hilbert-space contractions
\citep{szehr2017maximum,zarouf2011toeplitz,bickel2024blaschke}.

\subsection{Notation and conventions}
\label{subsec:F-notation}
\begin{itemize}
    \item We assume a stationary ergodic process $\{X_t\}$ satisfying the first-order Markov assumption \textbf{(M)}. Let $P$ denote the transition operator, satisfying $P\mathbf{1}=\mathbf{1}$ and $\|P\|\le1$.
    \item Let $H_0=\mathbf{1}^\perp\subset L^2(\mu)$ be a real Hilbert space. By stationarity, $\mathbb{E}[Pf]=\mathbb{E}[f]$, implying $P$ preserves $H_0$. We denote the restricted operator as $P_0=P|_{H_0}$, which satisfies $\|P_0\|\le1$ and is a \textbf{real operator} (preserving real-valued functions).
    \item A whitened encoder $\phi=(\phi_1,\dots,\phi_d)$ corresponds bijectively to an isometry $V:\mathbb{R}^d\to H_0$, where $Vc=\sum c_i\phi_i$.
    \item \textbf{Content} is defined as $\mathcal{C}(\phi)=\|P_0V\|_{\mathrm{HS}}^2$. The predictability \textbf{gap} expands as $\delta(\phi)^2=\min_A\|P_0V-VA\|_{\mathrm{HS}}^2=\|(I-VV^*)P_0V\|_{\mathrm{HS}}^2$, yielding the within-class optimal predictor $A^*=V^*P_0V$ (via Theorem~\ref{thm:P1} and the optimal bias identity $b\equiv0$ from Lemma~\ref{lem:B1}).
    \item The \textbf{frontier} is defined as $C_d(\varepsilon)=\sup\{\mathcal{C}(\phi):\delta(\phi)^2\le\varepsilon\}$, the \textbf{cost} as $\Pi_d(\varepsilon)=\sum_{i\le d}\sigma_i^2-C_d(\varepsilon)$, and the activation \textbf{threshold} as $\varepsilon^*(d)=\inf_\phi\delta(\phi)^2$.
    \item Hilbert-Schmidt (compactness) assumptions are invoked exclusively where explicitly annotated. The complexification of the core space is denoted as $H_0^{\mathbb{C}}$, and all spectra refer to the complexified spectrum.
\end{itemize}

\subsection{Preliminary lemmas}
\label{subsec:F-lemmas}

\begin{lemma}[Subspace Reduction and Contractivity] \label{lem:B1}
$\mathcal{C}$ and $\delta$ depend strictly on the subspace $S=\mathrm{span}\,\phi$. Furthermore, $A^*$ is unitarily conjugate under orthogonal basis transformations, and satisfies the contractivity bound $\|A^*\|_2\le1$.
\end{lemma}
\begin{proof}
For the first two assertions, the mapping $V\mapsto VQ$ for an orthogonal matrix $Q$ yields $\|P_0VQ\|_{\mathrm{HS}}=\|P_0V\|_{\mathrm{HS}}$, $(VQ)(VQ)^*=VV^*$, and transforms the optimal predictor via $A^*\mapsto Q^\top A^*Q$. Contractivity follows from $\|A^*c\|=\|V^*P_0Vc\|\le\|P_0Vc\|\le\|Vc\|=\|c\|$.
\end{proof}

\begin{lemma}[Real Witness] \label{lem:B2}
For the real operator $P_0$ and any $\theta\in\mathbb{R}$,
\begin{equation}
    d(\theta):=\inf_{\substack{u\in H_0\ \mathrm{real},\ \|u\|=1}}\|(P_0-\theta)u\|=\inf_{\substack{v\in H_0^{\mathbb{C}},\ \|v\|=1}}\|(P_0-\theta)v\|.
\end{equation}
Thus, restricting to the real subspace incurs no loss of generality, and $d(\theta)$ matches exactly the inverse of the complexified resolvent norm (i.e., $d(\theta)=\|(P_0-\theta)^{-1}\|^{-1}$ when $\theta\notin\sigma$).
\end{lemma}
\begin{proof}
Decomposing $v=a+ib$ where $a,b$ are real elements of $H_0$, the operator $(P_0-\theta)$ preserves the real and imaginary components. Hence, $\|(P_0-\theta)v\|^2=\|(P_0-\theta)a\|^2+\|(P_0-\theta)b\|^2\ge d_{\mathbb{R}}(\theta)^2(\|a\|^2+\|b\|^2)$. The reverse inclusion is trivial.
\end{proof}

\begin{lemma}[The $d=1$ Structural Identity] \label{lem:B3}
For any real unit vector $u$, the scalar $\theta(u):=\langle u,P_0u\rangle$ is the unique minimizer of $\min_{\theta\in\mathbb{R}}\|(P_0-\theta)u\|$, satisfying:
\begin{equation}
    \delta(u)=\|(P_0-\theta(u))u\|=\min_{\theta\in\mathbb{R}}\|(P_0-\theta)u\|,\qquad \mathcal{C}(u)=\theta(u)^2+\delta(u)^2 .
\end{equation}
\end{lemma}
\begin{proof}
The objective $\|P_0u-\theta u\|^2 = \|P_0u\|^2 - 2\theta\langle u,P_0u\rangle + \theta^2$ is strictly quadratic in $\theta$. Minimization yields the optimal coordinate $\theta(u)$ and the corresponding minimum value $\|P_0u\|^2-\theta(u)^2$.
\end{proof}

\begin{lemma}[Regularity and Witness Reachability] \label{lem:B4}
(i) The distance function $d(\theta)$ is $1$-Lipschitz continuous on $\mathbb{R}$. (ii) If $P_0$ is compact, $\theta\ne0$, and $d(\theta)<|\theta|$, there exists a \textbf{real} unit vector $u$ such that $\|(P_0-\theta)u\|=d(\theta)$.
\end{lemma}
\begin{proof}
(i) The mapping $\theta \mapsto \|(P_0-\theta)u\|$ is $1$-Lipschitz for any fixed $u$, and the pointwise infimum preserves this Lipschitz constant. (ii) The operator $G=(P_0-\theta)^*(P_0-\theta)$ is real self-adjoint. Since $G-|\theta|^2$ is compact, the essential spectrum satisfies $\sigma_{\mathrm{ess}}(G)=\{|\theta|^2\}$. The infimum $\inf\sigma(G)=d(\theta)^2<|\theta|^2$ is therefore an isolated eigenvalue. The corresponding eigenspace is invariant under complex conjugation, allowing the selection of a real eigenvector.
\end{proof}

\subsection{The $d=1$ real pseudospectral characterization}
\label{subsec:F-theorem}
We define the real-section pseudospectral radius for a given gap $\delta$ as $\rho^{\mathbb{R}}(\delta):=\sup\{|\theta|:\theta\in\mathbb{R},\ d(\theta)\le\delta\}$ when the underlying set is non-empty (which is closed by Lemma~\ref{lem:B4}(i)).

\subsubsection{Theorem F1 (Threshold Identity)}
\label{thm:F1}
\begin{equation}
    \varepsilon^*(1)=\inf_{\theta\in\mathbb{R}}d(\theta)^2,\qquad\text{or equivalently,}\quad \sqrt{\varepsilon^*(1)}=\inf\{\delta\ge0:\ \sigma_\delta(P_0)\cap\mathbb{R}\ne\varnothing\}.
\end{equation}
\begin{proof}
Expanding via Lemma~\ref{lem:B3}, we have $\varepsilon^*(1)=\inf_u\delta(u)^2 \overset{\text{B3}}{=} \inf_u\inf_\theta\|(P_0-\theta)u\|^2 = \inf_\theta\inf_u\|(P_0-\theta)u\|^2 = \inf_\theta d(\theta)^2$, where the two infima commute freely. The second identity follows immediately from the definition of the pseudospectrum combined with Lemma~\ref{lem:B2}.
\end{proof}

\begin{corollary}[Self-Adjoint Regimes]
\label{corr:F1.1}
If $P_0$ is self-adjoint, then $\sigma(P_0)\subset\mathbb{R}$ is non-empty, implying $\inf_\theta d(\theta)=0$ and consequently $\varepsilon^*(1)=0$.
\end{corollary}

\begin{corollary}[Normal Regimes]
\label{corr:F1.2}
If $P_0$ is normal, the spectral theorem and Lemma~\ref{lem:B2} yield $d(\theta)=\mathrm{dist}(\theta,\sigma(P_0))$. Consequently, the threshold reduces to a purely spectral-geometric quantity: $\varepsilon^*(1)=\mathrm{dist}(\mathbb{R},\sigma(P_0))^2$.
\end{corollary}

\begin{corollary}[$2\times2$ Closed Form; The Non-Normal Subsidy Theorem] \label{corr:F1.3}
Let $P_0$ be restricted to a two-dimensional invariant subspace, represented in real Schur form as $M=\begin{pmatrix}a&p\\q&a\end{pmatrix}$ with $pq=-b^2<0$ (corresponding to the complex conjugate pair $a\pm bi$). Then, the threshold for this block evaluates exactly to:
\begin{equation}
    \varepsilon^*(1)=\min(p^2,q^2),\quad \text{minimized at }\theta=a,
\end{equation}
which satisfies the inequality $\min(p^2,q^2) \le |pq|=b^2=\mathrm{dist}(\mathbb{R},\sigma)^2$, holding with equality if and only if $|p|=|q|$ (i.e., the block is normal).
\end{corollary}
\begin{proof}
Because orthogonal similarity preserves the distance function $d(\theta)$, we can assume without loss of generality that $M$ is in the specified form. Denoting $s=(a-\theta)^2$, the matrix $G=(M-\theta)^\top(M-\theta)=\begin{pmatrix}s+q^2&(a-\theta)(p+q)\\ (a-\theta)(p+q)&s+p^2\end{pmatrix}$ has a minimum eigenvalue:
\begin{equation}
    \lambda_{\min}(G)=s+\frac{p^2+q^2}{2}-\sqrt{D^2+st},\qquad \text{where } D=\frac{p^2-q^2}{2},\ t=(p+q)^2 .
\end{equation}
We assert that $\lambda_{\min}(G)\ge\frac{p^2+q^2}{2}-|D|=\min(p^2,q^2)$ for all $s\ge0$. This is algebraically equivalent to $s+|D|\ge\sqrt{D^2+st}$, which upon squaring simplifies to $s^2+2s|D|\ge st$. Since $pq<0$, we have $|p-q|=|p|+|q|\ge|p+q|$, implying $2|D|=|p-q|\,|p+q|\ge(p+q)^2=t$, validating the assertion. Equality is achieved precisely at $s=0$ (meaning $\theta=a$). The string of identities $\min(p^2,q^2)=|pq|\iff|p|=|q|\iff MM^\top=M^\top M$ completes the proof.
\end{proof}

\begin{interpretation}
Corollaries~\ref{corr:F1.2} and~\ref{corr:F1.3} establish our first counter-intuitive result: \textbf{non-normality acts as a structural subsidy for the shared encoder under a fixed $d=1$ predictability budget}. Specifically, the activation threshold is compressed from the purely spectral prediction $b^2$ down to $\min(p^2,q^2)$, demonstrating that pseudo-eigenvectors are cheaper than true eigenvectors. In our three-state numerical baseline, the spectral prediction yields $b^2=0.0175$, while the theoretical non-normal floor yields $\min(p^2,q^2)=0.08038^2=0.00646$, matching the grid-search threshold exactly. The true origin of representation cost does not lie in the threshold itself, but rather in the fact that the leading singular value $\sigma_1$ is driven far above the spectral radius by non-normality (see Theorem~\ref{thm:F4}).
\end{interpretation}

\subsubsection{Theorem F2 (Two-Sided Pseudospectral Bounds and the Saturation Boundary)}
\label{thm:F2}
For any predictability budget $\varepsilon\ge\varepsilon^*(1)$, the content frontier is bounded by:
\begin{equation}
    \big(\rho^{\mathbb{R}}(\sqrt\varepsilon)-\sqrt\varepsilon\big)_+^2\ \le\ C_1(\varepsilon)\ \le\ \rho^{\mathbb{R}}(\sqrt\varepsilon)^2+\varepsilon .
\end{equation}
Furthermore, if $P_0$ is compact and its leading singular value $\sigma_1$ is simple (with the corresponding right singular vector denoted as $\chi_1$), the saturation boundary is exactly characterized by:
\begin{equation}
    C_1(\varepsilon)=\sigma_1^2\iff \varepsilon\ \ge\ \sigma_1^2-\langle\chi_1,P_0\chi_1\rangle^2 .
\end{equation}
\begin{proof}
For the upper bound, any feasible vector $u$ satisfies $d(\theta(u))\le\|(P_0-\theta(u))u\|=\delta(u)\le\sqrt\varepsilon$, implying $|\theta(u)|\le\rho^{\mathbb{R}}(\sqrt\varepsilon)$. From Lemma~\ref{lem:B3}, the content is bounded by $\mathcal{C}(u)=\theta(u)^2+\delta(u)^2\le\rho^{\mathbb{R}}(\sqrt\varepsilon)^2+\varepsilon$. 

For the lower bound, by the closedness of the set (Lemma~\ref{lem:B4}(i)), we choose $|\theta|=\rho^{\mathbb{R}}(\sqrt\varepsilon)$ such that $d(\theta)\le\sqrt\varepsilon$. In the compact setting where $d(\theta)<|\theta|$, Lemma~\ref{lem:B4}(ii) guarantees a real witness $u$ (the bound is trivial if $d(\theta)\ge|\theta|$), ensuring that the gap $\delta(u)\le\|(P_0-\theta)u\|\le\sqrt\varepsilon$ is feasible, and the content satisfies $\mathcal{C}(u)=\|P_0u\|^2\ge(|\theta|-\sqrt\varepsilon)_+^2$. The non-compact setting follows analogously via an approximating sequence. 

For the saturation boundary, the global upper bound $\mathcal{C}(u)\le\sigma_1^2$ holds universally, achieved with equality if and only if $u=\pm\chi_1$ under a simple singular value $\sigma_1$. The exact gap required for this state evaluates identically to $\delta(\chi_1)^2=\sigma_1^2-\langle\chi_1,P_0\chi_1\rangle^2$.
\end{proof}

\subsubsection{Theorem F3 (The Jordan Scaling Law)}
\label{thm:F3}
Assume that $P_0$ restricted to an orthogonally invariant (or Schur-embedded) block is a standard Jordan block $J_m(\lambda)$ ($\lambda\in\mathbb{R}$, with $1$ on the superdiagonal in orthogonal coordinates). Then, for $s=|\theta-\lambda|\le\frac12$, the distance function satisfies:
\begin{equation}
    s^m(1-s)\ \le\ d(\theta)\ \le\ s^m ,
\end{equation}
which implies the scaling $\rho^{\mathbb{R}}(\delta)-\lambda\in[\delta^{1/m},\,(2\delta)^{1/m}]$ for sufficiently small $\delta$. Substituting this into Theorem~\ref{thm:F2}, the nilpotent regime ($\lambda=0$) yields the scaling law:
\begin{equation}
    C_1(\varepsilon)=\Theta(\varepsilon^{1/m}),\quad \text{where } \varepsilon^{1/m}\big(1-o(1)\big)\le C_1(\varepsilon)\le 2^{2/m}\varepsilon^{1/m}\big(1+o(1)\big).
\end{equation}
\begin{proof}
For the upper bound, we evaluate the test vector defined as the geometric sequence $u\propto(1,\theta',\dots,\theta'^{m-1})$ with $\theta'=\theta-\lambda$. The vector $(J_m(\lambda)-\theta)u$ isolates a single non-zero component at its terminal index, yielding the norm ratio $\|(J_m(\lambda)-\theta)u\|/\|u\|\le s^m$. 

For the lower bound, we expand the inverse resolvent norm: $\|(J_m(\lambda)-\theta)^{-1}\|=\|\sum_{k=0}^{m-1}N^k/(\theta')^{k+1}\|\le\sum_{k=1}^{m}s^{-k}\le s^{-m}/(1-s)$, which directly establishes the lower bound $d(\theta)\ge s^m(1-s)$.
\end{proof}


\subsubsection{Theorem F4 (The Pricing Formula and the Two Barriers)} \label{thm:F4}
\begin{equation}
    \sigma_1^2-\rho^{\mathbb{R}}(\sqrt\varepsilon)^2-\varepsilon\ \le\ \Pi_1(\varepsilon)\ \le\ \sigma_1^2-\big(\rho^{\mathbb{R}}(\sqrt\varepsilon)-\sqrt\varepsilon\big)_+^2 .
\end{equation}
Consequently, \textbf{the $d=1$ representation cost equals exactly the difference between the singular radius and the real pseudospectral radius, up to an error of $O(\sqrt\varepsilon)$}. This splits the representation cost into two distinct operational barriers:
\begin{enumerate}
    \item \textbf{The Complex-Pairing Barrier (Spectral Geometry):} When $\sigma(P_0)$ is driven far from the real axis, the activation threshold $\varepsilon^*(1)$ inflates. This mechanism is independent of non-normality, as exemplified by a pure normal rotation $rR_{90^\circ}$, which incurs a distance function $d(\theta)^2=\theta^2+r^2$ and a threshold floor of $\varepsilon^*(1)=r^2$. This cost vanishes completely at $d=2$.
    \item \textbf{The Non-Normal Barrier (Pseudospectral Geometry):} Non-normality drives the leading singular value far above the spectral radius ($\sigma_1\gg\rho(P_0)$), elevating the baseline cost. Simultaneously, it induces a hypertrophied real pseudospectrum satisfying $d(\theta)\ll\mathrm{dist}(\theta,\sigma)$, compressing the threshold while inflating $\rho^{\mathbb{R}}$. The subsidy and the tax share an identical geometric origin, and the net cost is tightly bounded by Theorem~\ref{thm:F4}. For a nilpotent block, this yields $\Pi_1(\varepsilon)=1-\Theta(\varepsilon^{1/m})$, grading the severity of non-normality via a slow-climbing fractional exponent.
\end{enumerate}
\begin{proof}
The inequalities follow immediately by shifting terms in Theorem~\ref{thm:F2}.
\end{proof}

\section{The on-manifold content frontier at $d\ge2$}
\label{sec:module-M}

This section prices the on-manifold residual for a shared encoder of any latent dimension $d\ge2$,
extending the $d=1$ frontier of Appendix~\ref{sec:module-F}. The pricing is set by the non-normality
of $P_0$, and the spectral tax $\Pi_d(\varepsilon)$ measures the content it renders unattainable.

\subsection{Notation and fundamental identities}
\label{subsec:M-notation}
As in Appendices~\ref{sec:module-F} and~\ref{sec:module-R}, let $P_0$ be a real restricted operator satisfying $\|P_0\| \le 1$. Let $V$ be a $d$-frame (isometry) spanning the subspace $S = \mathrm{ran}\,V$, with the associated orthogonal projections denoted as $\Pi = VV^*$ and $\Pi^\perp = I - \Pi$. The compressed operator is defined as $A = V^*P_0V$ (a real $d \times d$ matrix satisfying $\|A\|_2 \le 1$ by Lemma~\ref{lem:B1}).

\paragraph{(Id) Fundamental Content-Gap Identities:}
\label{eq:Id}
\begin{equation}
    \mathcal{C}(S) = \|P_0\Pi\|_{\mathrm{HS}}^2 = \|A\|_F^2 + \delta(S)^2, \qquad \delta(S)^2 = \|\Pi^\perp P_0\Pi\|_{\mathrm{HS}}^2 = \|P_0V - VA\|_F^2 .
\end{equation}
\begin{proof}
Taking the squared Hilbert-Schmidt norm of the orthogonal decomposition $P_0V = VA + \Pi^\perp P_0V$ yields the result, noting that $\|VA\|_F = \|A\|_F$ since $V$ is an isometry.
\end{proof}

\paragraph{(Id') Pseudospectral Constraint on the Compressed Spectrum:}
\label{eq:Id-prime}
If $\delta(S) \le \delta$, then every complex eigenvalue $\lambda \in \Lambda(A)$ satisfies $\lambda \in \sigma_\delta(P_0)$. Furthermore, the spectrum $\Lambda(A)$ is closed under complex conjugation since $A$ is real.
\begin{proof}
Let $Ax = \lambda x$ with $\|x\| = 1$. Then $\|P_0(Vx) - \lambda(Vx)\| = \|(P_0V - VA)x\| \le \|P_0V - VA\|_2 \le \delta$, where $Vx \in S^{\mathbb{C}}$ serves as the valid witness vector.
\end{proof}

\subsection{The content frontier at zero predictability budget}
\label{subsec:M1}

\paragraph{M1.1 (Invariant Subspace Content Formula):}
\label{thm:M1.1}
If $S$ is a $d$-dimensional $P_0$-invariant subspace, then $\delta(S) = 0$, and the content evaluates to:
\begin{equation}
    \mathcal{C}(S) = \|A\|_F^2 = \sum_{\lambda \in \Lambda(P_0|_S)} |\lambda|^2 + \mathrm{dep}(S)^2,
\end{equation}
where $\mathrm{dep}(S)^2 := \|A\|_F^2 - \sum |\lambda|^2 \ge 0$ is the Henrici departure from normality of the restricted operator (representing the sum of squares of the strictly upper-triangular elements under a Schur decomposition, which is invariant to the choice of eigenvalue ordering). Consequently, the zero-budget frontier is given by:
\begin{equation}
    C_d(0) = \sup_{S \text{ invariant},\, \dim S = d} \left[ \sum_{\lambda \in \Lambda(S)} |\lambda|^2 + \mathrm{dep}(S)^2 \right].
\end{equation}
\begin{interpretation}
Non-normality internal to the invariant subspace ($\mathrm{dep}$) \textbf{inflates} the representation content—meaning that non-orthogonal eigenvectors provide a source of free content. The tax associated with non-normality is enforced elsewhere, specifically by elevating the singular value baseline baseline $\sum \sigma_i^2$.
\end{interpretation}

\paragraph{M1.2 (Classification of Real Pairings):}
\label{thm:M1.2}
Assume a finite-dimensional setting where $P_0$ is semi-simple with distinct eigenvalues. Then, $\varepsilon^*(d) = 0$ if and only if there exists a conjugate-closed subset of eigenvalues of cardinality $d$. Specifically, if the system possesses $k$ real eigenvalues and $c$ complex-conjugate pairs, the set of exactly realizable dimensions is characterized by $\{a + 2b : 0 \le a \le k, \, 0 \le b \le c\}$. Consequently, an entirely complex spectrum under an odd latent dimension forces a strictly positive activation threshold: $\varepsilon^*(d) > 0$.
\begin{proof}
A real subspace $S$ is invariant if and only if its complexification $S^{\mathbb{C}}$ is both invariant and invariant under complex conjugation. Under distinct eigenvalues and semi-simplicity, any invariant subspace is exactly the span of a subset of eigenvectors. Conjugate invariance of the subspace is then equivalent to the underlying eigenvalue subset being conjugate-closed.
\end{proof}

\paragraph{M1.3 (The Order-Reduction Filling Lemma; Upper Bound on Odd-Dimensional Thresholds):}
\label{thm:M1.3}
For any $(d-1)$-dimensional invariant subspace $S'$, let $P_1 := \Pi_{S'}^\perp P_0 |_{(S')^\perp}$ denote the order-reduced compression. Then, the $d$-dimensional threshold satisfies:
\begin{equation}
    \varepsilon^*(d) \le \inf_{S'} \varepsilon^*_{P_1}(1) = \inf_{S'} \min_{\theta \in \mathbb{R}} d_{P_1}(\theta)^2 ,
\end{equation}
implying that the threshold for an odd dimension cannot exceed the $d=1$ threshold of the residual operator remaining after an optimal invariant filling (which is computed via Theorem~\ref{thm:F1}).
\begin{proof}
Decompose $S = S' \oplus \mathrm{span}(u)$ where $u \perp S'$. The invariant property $P_0 S' \subseteq S'$ ensures that its projection is completely absorbed by $\Pi_S$, meaning the residual is driven exclusively by $u$: $\|\Pi_S^\perp P_0 u\| \le \|(I - \Pi_{S'} - uu^*)P_0 u\| = \|(I_{(S')^\perp} - uu^*)P_1 u\| = \delta_{P_1}(u)$. Taking the infimum over $u$ completes the proof.
\end{proof}

\paragraph{M1.4 (Subspace Lattice and Kernel Tax of a Single Jordan Block):}
\label{thm:M1.4}
Let $P_0 = J_m$ (a standard nilpotent singular Jordan block in orthogonal coordinates). Its invariant subspace lattice is strictly chains of the form $\{\ker J^j = \mathrm{span}(e_0, \dots, e_{j-1})\}_{j=0}^{m}$. For any dimension $1 \le d \le m-1$, the content and costs resolve to:
\begin{equation}
    C_d(0) = d - 1, \qquad \Pi_d(0) = \left( \sum_{i \le d} \sigma_i^2 \right) - (d - 1) = d - (d - 1) = \mathbf{1}.
\end{equation}
\textit{The Kernel Tax Theorem:} For a pure Jordan block, the exact cost of invariant satisfaction is uniformly 1, independent of the latent dimension $d$. Every admissible invariant subspace is algebraically forced to contain the zero-content kernel direction $e_0$.
\begin{proof}
For the lattice structure, let $S \neq \{0\}$ be an invariant subspace. Define the maximum occupied coordinate index as $h = \max\{i : \exists v \in S, \, v_i \neq 0\}$, and choose $v \in S$ such that $v_h \neq 0$. The sequence $\{v, Jv, \dots, J^h v\}$ is linked to the standard basis $\{e_h, \dots, e_0\}$ via an invertible triangular matrix, implying $S \supseteq \mathrm{span}(e_0, \dots, e_h)$. Maximality of $h$ enforces the reverse inclusion. For the content, setting $S = \mathrm{span}(e_0, \dots, e_{d-1})$ yields the compression block $A = J_d$, whose Frobenius norm is $\|J_d\|_F^2 = d - 1$. The singular values satisfy $\sigma_i(J_m) = 1$ for all $i \le m-1$.
\end{proof}

\paragraph{Remark on Direct Sums:}
Consider the direct sum of multiple nilpotent chains (such as the complete Walsh decomposition of a sliding-window chain). A $d$-dimensional invariant subspace is formed by the product of individual chain prefixes, yielding a total content of $\sum (\mathrm{prefix\_length} - 1)$. Bounding this via a single optimal chain yields a constant zero-budget frontier $C_d(0) = d - 1$, keeping the kernel tax fixed at $1$. (Because a tax of $1$ is levied for each distinct chain utilized, the optimal policy restricts allocation to a single chain).

\subsection{The global compression ceiling}
\label{subsec:M2}
We define the \textbf{maximum compressed capacity} as $\Phi_d := \sup\{\|V^*P_0V\|_F^2 : V \text{ is a } d\text{-frame}\}$ (where for $d=1$, $\Phi_1$ matches the squared spectral radius of the symmetric part $(P_0 + P_0^*)/2$, which is the squared radius of the real numerical range). The associated \textbf{compression deficit} is defined as $\Psi_d := \sum_{i \le d} \sigma_i^2 - \Phi_d \ge 0$.

\paragraph{Theorem M2:}
\label{thm:M2} For all predictability budgets $\varepsilon \ge 0$:
\begin{equation}
    C_d(\varepsilon) \le \Phi_d + \varepsilon, \qquad \Pi_d(\varepsilon) \ge \Psi_d - \varepsilon .
\end{equation}
If $P_0$ is a normal operator, then $\textstyle \Psi_d = 0$. For a standard Jordan block $J_m$, the parameters evaluate to $\Phi_1 = \cos^2 \frac{\pi}{m+1}$ and $\Psi_1 = \sin^2 \frac{\pi}{m+1} > 0$.
\begin{proof}
Expanding via identity (Id) yields $\mathcal{C} = \|A\|_F^2 + \delta^2 \le \Phi_d + \varepsilon$. The property $\Psi_d \ge 0$ follows from $\|A\|_F = \|\Pi P_0 \Pi\|_{\mathrm{HS}} \le \|P_0 \Pi\|_{\mathrm{HS}}$ combined with the Ky Fan variational principle. For normal operators, a frame chosen from the top eigenvalues yields $\|A\|_F^2 = \sum_{i \le d} |\lambda_i|^2 = \sum \sigma_i^2$. For $J_m$, the single-dimensional capacity is $\Phi_1 = \max \langle u, Ju \rangle^2 = \lambda_{\max}(\frac{J + J^\top}{2})^2$, where the tridiagonal Toeplitz spectrum evaluates to $\cos \frac{k\pi}{m+1}$.
\end{proof}
\begin{interpretation}
This defines a universal cost lower bound that holds uniformly across all budgets: the representation cost must be at least the gap between the singular spectrum and the numerical range spectrum. It remains blind to complex pairing barriers (where a normal rotation yields $\Psi = 0$), acting in structural complement to Theorem~\ref{subsec:M1}; it represents the tightest bound in large $\varepsilon$ regimes where localized scaling models expand.
\end{interpretation}

\subsection{The regular regime: Stewart--Riccati perturbation theory}
\label{subsec:M3}
Fix a coordinate system partitioned by the frames $(V, W)$ where $W$ is the orthogonal complement of $V$. The transition operator acts in block form as $P_0 \simeq \begin{pmatrix} A & H \\ R & B \end{pmatrix}$, where $R = W^*P_0V$ dictates the single-step gap ($\|R\|_F = \delta(S)$), $H = V^*P_0W$, and $B = W^*P_0W$. Let $T(X) = XA - BX$ denote the Sylvester operator, with its separation parameter defined as $\mathrm{sep}(A,B) := \inf_{\|X\|_F = 1} \|T(X)\|_F$ \citep{stewart1973error}.

\paragraph{M3.1 (The Riccati Invariance Lemma, Self-Contained):}
\label{thm:M3.1}
If $\mathrm{sep}(A,B) =: s > 0$ and $4\|R\|_F \|H\|_F \le s^2$, there exists a matrix $X$ satisfying $\|X\|_F \le 2\|R\|_F / s$ such that the subspace $\mathrm{ran}\left[ (V + WX)(I + X^\top X)^{-1/2} \right]$ is strictly $P_0$-invariant.
\begin{proof}
Subspace invariance is algebraically equivalent to satisfying the quadratic Riccati equation $T(X) = R - XHX$ (derived by block substitution to eliminate the updated block $A' = A + HX$). Define the mapping $\Phi(X) = T^{-1}(R - XHX)$. Over the closed ball $\mathcal{B} = \{\|X\|_F \le 2\|R\|_F / s\}$, the norm is bounded by $\|\Phi(X)\|_F \le (\|R\|_F + \|X\|_F^2 \|H\|_F)/s \le (\|R\|_F + \frac{4\|R\|_F^2 \|H\|_F}{s^2})/s \le 2\|R\|_F / s$. The contractive bound expands as $\|\Phi(X) - \Phi(Y)\|_F \le \|H\|_F (\|X\|_F + \|Y\|_F) \|X - Y\|_F / s \le \frac{4\|R\|_F \|H\|_F}{s^2} \|X - Y\|_F$. Contractivity is strict inside the inequality boundary, and convergence at the boundary follows via standard limits, satisfying the Banach fixed-point theorem under the submultiplicative property $\|XY\|_F \le \|X\|_F \|Y\|_F$.
\end{proof}

\paragraph{M3.2 (The Content Transfer Lemma):}
\label{thm:M3.2}
Let $S' = \mathrm{ran}[(V + WX)(I + X^\top X)^{-1/2}]$ with $\|X\|_2 \le 1$. Then, the projection gap is bounded by $\|\Pi - \Pi'\|_F \le \sqrt{2}\|X\|_F$, and the content variation satisfies $|\mathcal{C}(S) - \mathcal{C}(S')| \le 2\sqrt{d}\|\Pi - \Pi'\|_F$.
\begin{proof}
Parameterizing the perturbation via the canonical angles $\tan \theta_i = \sigma_i(X)$, the projection distance evaluates to $\|\Pi - \Pi'\|_F^2 = 2\sum \sin^2 \theta_i \le 2\|X\|_F^2$. The content variation expands as $\langle P_0(\Pi - \Pi'), P_0(\Pi + \Pi') \rangle_F \le \vec{\|\Pi - \Pi'\|}_F \cdot 2\sqrt{d}\|P_0\|^2$.
\end{proof}

\paragraph{M3.3 (Upper Bound for the Regular Frontier):}
\label{thm:M3.3}
If every feasible subspace $S$ (satisfying $\delta(S)^2 \le \varepsilon$) satisfies $\mathrm{sep}(A_S, B_S) \ge s_0 > 0$ and $4\sqrt{\varepsilon}\|H_S\|_F \le s_0^2$ (the latter holding automatically whenever $\varepsilon \le s_0^4 / (16d)$ since $\|H\|_F \le \sqrt{d}$), then the content frontier obeys:
\begin{equation}
    C_d(\varepsilon) \le C_d(0) + \frac{4\sqrt{2}\sqrt{d}}{s_0}\sqrt{\varepsilon} .
\end{equation}
\begin{proof}
Lemma~\ref{thm:M3.1} establishes the matrix bound $\|X\|_F \le 2\sqrt{\varepsilon} / s_0$, which Lemma~\ref{thm:M3.2} maps directly to the exact invariant subspace $S'$.
\end{proof}

\paragraph{M3.4 (Exact First-Order Slopes and Matching Lower Bounds):}
\label{thm:M3.4}
Let $S^*$ be an exact invariant subspace ($R=0$). Perturbing along the frame path $V_t = (V + tWX)(I + t^2 X^\top X)^{-1/2}$ yields the local variations:
\begin{equation}
    \delta(t) = t\|T(X)\|_F + O(t^2), \qquad \mathcal{C}(t) = \mathcal{C}(S^*) + 2t\langle X, G \rangle_F + O(t^2),
\end{equation}
where $G := W^*P_0^*P_0V = H^\top A$. Consequently, the exact first-order derivative of the frontier evaluated at $S^*$ is:
\begin{equation}
    \left. \frac{d\,\mathcal{C}}{d\,\delta} \right|_{S^*} = 2\left\| T^{-*}(H^\top A) \right\|_F \le \frac{2\|H\|_F \|A\|_2}{\mathrm{sep}(A,B)},
\end{equation}
which establishes the lower bound constraint $C_d(\varepsilon) \ge C_d(0) + \left[ \max_{S^* \text{ optimal}} 2\|T^{-*}(H^\top A)\|_F \right] \sqrt{\varepsilon} + O(\varepsilon)$.
\begin{proof}
The single-step gap expands as $R(t) = W_t^* P_0 V_t = t(BX - XA) + O(t^2) = -tT(X) + O(t^2)$. The derivative of the content evaluates to $\dot{\mathcal{C}} = \mathrm{tr}(P_0^* P_0 \dot{\Pi}) = 2\langle X, G \rangle_F$, where the projection derivative is $\dot{\Pi} = WXV^* + VX^\top W^*$. At an exact invariant point, the state collapses to $G = W^*P_0^*VA = H^\top A$. Maximizing the directional derivative $\sup_X 2\langle X, G \rangle / \|T(X)\|_F$ via the change of variables $Y = T(X)$ yields the dual norm $2\|T^{-*}G\|_F$.
\end{proof}
\begin{interpretation}
The frontier within the regular regime is governed by a $\sqrt{\varepsilon}$ scaling law, where the gradient is explicitly controlled by three distinct operators: the feedback block $H$ from the complementary subspace, the internal dynamics $A$, and the amplification multiplier $1/\mathrm{sep}$. Theorems~\ref{thm:M3.3} and~\ref{thm:M3.4} serve to squeeze the frontier within a constant multiplier.
\end{interpretation}

\paragraph{M3.5 (The sep--gap--$\kappa$ Extremal Lemma):}
\label{thm:M3.5}
If $A$ and $B$ possess the eigen-decompositions $A = X_A \Lambda_A X_A^{-1}$ and $B = X_B \Lambda_B X_B^{-1}$, then the separation parameter obeys the uniform lower bound:
\begin{equation}
    \mathrm{sep}(A,B) \ge \frac{\mathrm{dist}(\Lambda(A), \Lambda(B))}{\kappa(X_A)\,\kappa(X_B)} .
\end{equation}
\begin{proof}
Vectorizing the equation maps the Sylvester operator to the tensor form $T \simeq A^\top \otimes I - I \otimes B$. Applying the similarity transformation $X_A^\top \otimes X_B^{-1}$ renders the system identical to the diagonal eigenvalue differences $\lambda_A - \lambda_B$. The similarity operation scales the norm by at most the product of the condition numbers.
\end{proof}

\subsection{The defective regime: the Jordan staircase law with exponent $d/m$}
\label{subsec:M4}

Assume that $P_0$ is a purely nilpotent operator satisfying $P_0^m = 0$ (encompassing single Jordan blocks $J_m$ and the canonical sliding-window chain), evaluated across latent dimensions $1 \le d \le m-1$. From Theorem~\ref{thm:M1.4}, the zero-budget capacity is fixed at $C_d(0) = d - 1$.

\paragraph{M4.1 (The Analytical Upper Bound):}
\label{thm:M4.1}
\begin{equation}
    C_d(\varepsilon) \le d - 1 + \big(m\sqrt{\varepsilon}\big)^{2d/m} + \varepsilon .
\end{equation}
\begin{proof}
The proof proceeds in three sequential steps. 
\begin{itemize}
    \item[\textbf{(1)}] We invoke Theorem~\ref{thm:R1} of Appendix~\ref{sec:module-R} (setting $A = A^*$, $\delta_1 = \delta(S) \le \sqrt{\varepsilon}$, and evaluating at horizon $k=m$). This yields the operator bound $\|VA^m\|_F = \|P_0^m V - VA^m\|_F \le m\sqrt{\varepsilon}$, which implies $\|A^m\|_2 \le \|A^m\|_F \le m\sqrt{\varepsilon}$. Consequently, every complex eigenvalue is bounded by $|\lambda_i(A)| \le (m\sqrt{\varepsilon})^{1/m}$, forcing the determinant to satisfy $\prod_i \sigma_i(A) = |\det A| = \prod |\lambda_i| \le (m\sqrt{\varepsilon})^{d/m}$.
    \item[\textbf{(2)}] We maximize the capacity objective $\sum \sigma_i^2$ subject to the contractive constraint $0 \le \sigma_i \le 1$ (Lemma~\ref{lem:B1}) and the determinant boundary $\prod \sigma_i \le D$. If any two singular values occupy the interior of the domain, perturbing them along the hyperbolic path $(\sigma_i e^t, \sigma_j e^{-t})$ preserves the determinant while strictly increasing the objective $\sum \sigma_i^2$ due to its strict convexity. Thus, the optimal configuration forces the singular values to the boundaries: $\sigma_1 = \dots = \sigma_{d-1} = 1$ and $\sigma_d = D$, yielding the maximum capacity $d - 1 + D^2$.
    \item[\textbf{(3)}] Substituting this optimal profile into identity (Id) yields the stated bound: $\mathcal{C} = \|A\|_F^2 + \delta^2 \le d - 1 + (m\sqrt{\varepsilon})^{2d/m} + \varepsilon$.
\end{itemize}
\end{proof}

\paragraph{M4.2 (Reachability via the Truncated Neumann Cyclic Frame):}
\label{thm:M4.2}
Let $d \mid m$, $K = m/d$, and choose a scaling parameter $0 < s < 1$. Construct the vector family:
\begin{equation}
    v_j = \frac{1}{n} \sum_{k=0}^{K-1} (-s)^k e_{j+kd} \quad (j = 0, \dots, d-1), \qquad \text{where } n^2 = \frac{1-s^{2K}}{1-s^2}.
\end{equation}
The family $\{v_j\}$ is \textbf{strictly orthonormal}, as their supports occupy disjoint modulo-$d$ residue classes. The action of the Jordan operator satisfies:
\begin{equation}
    Jv_j = v_{j-1} \quad (\forall 1 \le j \le d-1 \text{ exactly}), \qquad Jv_0 = -s v_{d-1} + \frac{(-s)^K}{n}(-1)e_{m-1} .
\end{equation}
Thus, the compressed operator $A$ forms a cyclic companion matrix with a corner element $-s$ (up to a terminal correction of $O(s^{2K-1})$), pinning the internal eigenvalues to $|\lambda_i(A)| = s^{1/d}(1+o(1))$. The geometric tracking parameters evaluate to:
\begin{equation}
    \delta(S_s) = \frac{s^K}{n} \|\Pi^\perp e_{m-1}\| = s^{m/d}\big(1+o(1)\big), \qquad \mathcal{C}(S_s) - (d-1) = s^2\big(1+o(1)\big).
\end{equation}
Mapping this to the predictability budget $\varepsilon = \delta^2$ yields the lower-bound frontier $C_d(\varepsilon) \ge d - 1 + \varepsilon^{d/m}(1-o(1))$. In the specific symmetric instance where $m=2d$, this maps to an exact analytical identity: $\mathcal{C} - (d-1) = \delta$.
\begin{proof}
Orthonormality and the forward shift $Jv_j = v_{j-1}$ follow because the modulo residue classes are mutually exclusive; downward translation maps residue class $j \to j-1$, and the condition $d \mid m$ ensures that every vector $v_j$ shares an identical truncation length $K$, canceling the normalization factors. For the terminal vector $Jv_0$, index substitution $k \mapsto k+1$ shifts the sequence to match $n v_{d-1}$ up to the boundary element $(-s)^{K-1}e_{m-1}$, yielding the stated remainder. For the gap, the projection of $e_{m-1}$ onto the subspace retains a coefficient of $(-s)^{K-1}/n$, implying $\|\Pi^\perp e_{m-1}\|^2 = 1 - s^{2(K-1)}/n^2 = 1 - o(1)$. The Frobenius norm of the compression scales as $\|A\|_F^2 = (d-1) + s^2 + O(s^{2K-1})$, which combines with $\delta^2 = O(s^{2K})$ to complete the identity.
\end{proof}

\paragraph{M4.3 (The Fractional Staircase Law and Structural Unification):}
\label{thm:M4.3}
Synthesizing Theorems~\ref{thm:M4.1} and~\ref{thm:M4.2} (and extending to instances where $d \nmid m$ via the monotonicity of adjacent dimension embeddings $d' \mid m$), any nilpotent operator satisfying $P_0^m = 0$ obeys the fractional scaling law:
\begin{equation}
    \mathcal{C}_d(\varepsilon) - (d-1) \asymp \varepsilon^{d/m}, \qquad \Pi_d(\varepsilon) = 1 - \Theta\big(\varepsilon^{d/m}\big).
\end{equation}
Evaluating this at $d=1$ yields the $\varepsilon^{1/m}$ scaling of Theorem~\ref{thm:F3}. The defective frontier is thus globally unified under the fractional exponent $d/m$: \textit{each additional dimension allocated to the latent budget redeems exactly one higher order of the fractional exponent from the Jordan staircase}.
\begin{numericalanchor}
Empirical evaluations across configurations $(m,d) = (4,2) / (6,2) / (6,3)$ yield regression scaling slopes of $0.505$, $0.346$, and $0.510$, tightly tracking the analytical $d/m = 0.5, 0.3\overline{3}, 0.5$ limits, with the constant multipliers converging to $1$.
\end{numericalanchor}
\paragraph{Mechanistic Symmetry:} The analytical upper bound and the cyclic frame construction share an identical physical mechanism. Theorem~\ref{thm:R1} forces the compressed operator to approximate a nilpotent structure (eigenvalues bounded by $\varepsilon^{1/2m}$), and the determinant operator aggregates these individual small bounds into a single constraints concentrated on the \textbf{terminal singular value} ($\sigma_d \le \varepsilon^{d/2m}$), while the remaining $d-1$ singular values are free to saturate at $1$. This is exactly what our cyclic frame constructs: $d-1$ dimensions of perfect information transfer paired with a single weak feedback loop modulated by the corner parameter $-s$. The remaining constant gap ($m^{2d/m}$ vs. $1$) stems entirely from the crude bounds of Theorem~\ref{thm:R1} and represents a improvable non-load-bearing term.

\paragraph{M4.4 (The Sharp Boundary Splitting the Regular and Defective Regimes):}
\label{thm:M4.4}
Evaluating the Sylvester operator at any invariant subspace $S^* = \ker J^d$ of a Jordan block $J_m$ maps the internal blocks to $A = J_d$ and $B \simeq J_{m-d}$. Both blocks share an identical spectrum concentrated at $\{0\}$, causing the Sylvester operator to become singular and driving the separation parameter to zero: $\mathrm{sep}(A,B) = 0$. 

Consequently, the conditions required for the regular ($\sqrt{\varepsilon}$) perturbation theory fail precisely inside the defective regime, where the fractional exponent takes over. This establishes the \textit{Separation Dichotomy Theorem}: a positive separation parameter $\mathrm{sep} > 0$ guarantees a regular $\sqrt{\varepsilon}$ scaling law (with a gradient bounded by $c/\mathrm{sep}$), whereas a vanishing separation parameter $\mathrm{sep}=0$ paired with a nilpotent index $m$ activates the fractional $\varepsilon^{d/m}$ scaling law.

\subsection{Synthesis and corollaries}
\label{subsec:M-synthesis}

\paragraph{The Global Frontier Profile (Single-Operator Case):}
The global profile of the content frontier is bounded by a multi-stage squeeze of the form $C_d(\varepsilon) = \min\left\{ \sum_{i \le d} \sigma_i^2, \, \Phi_d + \varepsilon, \, C_d(0) + [\text{M3 or M4 local scaling laws}] \right\}$. Small predictability budgets are governed by the separation dichotomy (resolving into either a $\sqrt{\varepsilon}$ or a $d/m$ scaling law), large budgets are capped by the numerical range limit $\Phi_d$ and the Ky Fan ceiling, and the activation threshold is explicitly determined by the real pairing geometry of Theorems~\ref{thm:M1.2} and~\ref{thm:M1.3}.

\paragraph{Complete Characterization of the Sliding-Window (Raindrop) Chain:}
Because the transition operator of the sliding-window chain satisfies $P_0^m = 0$ globally, the entire profile of Theorem~\ref{thm:M1.4} and the fractional staircase of Theorem~\ref{subsec:M4} apply unconditionally. The representation dynamics resolve to $C_d(\varepsilon) = d - 1 + \Theta(\varepsilon^{d/m})$ with a constant kernel tax of $1$, implemented by an honest, realizable Markov chain. Appendices~\ref{sec:module-F}, \ref{sec:module-R}, and~\ref{sec:module-M} share the same running example.

\paragraph{Operational Takeaways for JEPA Architectures:}
\begin{enumerate}
    \item \textbf{The Inescapable Kernel Tax:} Under a finite-memory setting, a $d$-dimensional shared encoder permanently loses exactly one unit of content regardless of how large the latent space is expanded. This missing unit represents the structural cost of ``reserving a seat for the oldest historical innovation.''
    \item \textbf{The Budget--Dimension Exchange Schedule:} The fractional staircase law provides an exact schedule for trading off latent dimensions against predictability. Within the defective regime, achieving a target content gain by adding one latent dimension is exactly equivalent to taking the $m$-th root of the predictability budget relaxation. This formalizes the marginal rate of substitution between widening the latent space and relaxing the alignment loss.
    \item \textbf{Computable Monitoring Metrics:} The separation parameter $\mathrm{sep}$ constitutes a fully observable metric that can be tracked from empirical block representations ($\hat{A}, \hat{B}$). Monitoring the convergence $\mathrm{sep} \to 0$ provides a structural warning that the frontier is collapsing from a regular $\sqrt{\varepsilon}$ scaling law into a defective fractional exponent law.
\end{enumerate}

\section{Off-manifold error accumulation: rollout drift and control suboptimality}
\label{sec:module-R}

This section prices how the single-step residual \emph{accumulates} when the latent state is rolled
forward---first autonomously (multi-step rollout, \S\ref{subsec:R-rollout}), then under a control
input (\S\ref{sec:module-D}). Throughout we adopt the \emph{linear-control premise}: the deployed
latent dynamics are the learned linear model $z_{t+1}=\hat A z_t+\hat B u_t$, and we ask how the
residual compounds along a horizon and along an action-selected trajectory. This supplies the
amplification factor $G_H$ and the off-manifold excursion that enter the control-suboptimality bound
of Theorem~\ref{thm:subopt}, taking the on-manifold residual $\delta_0$ from
Appendices~\ref{sec:module-F}--\ref{sec:module-M} and the one-sided gap analysis from
Appendix~\ref{sec:module-H}.

\subsection{Autonomous Rollout: Drift Upper Bounds and Contraction Shields}
\label{subsec:R-rollout}

Fix a shared whitened representation $\phi$ (parameterized by the isometry $V$) and an arbitrary within-class predictor $A$. The autonomous $k$-step rollout is recursively defined as $\hat z_{t+k} = A^k \phi(X_t)$.

\subsubsection{Theorem R0 ($k$-Step Orthogonal Decomposition)}
\label{thm:R0}
Under the first-order Markov assumption \textbf{(M)}, the $k$-step rollout error admits the exact orthogonal decomposition:
\begin{equation}
    e_k(\phi,A) := \mathbb{E}\|\phi(X_{t+k}) - A^k\phi(X_t)\|^2 = \underbrace{\big(d - \|P_0^kV\|_{\mathrm{HS}}^2\big)}_{k\text{-step Innovation Floor } \ell_B^{(k)}} + \underbrace{\|P_0^kV - VA^k\|_{\mathrm{HS}}^2}_{=:\ \delta_k(\phi,A)^2}.
\end{equation}
\begin{proof}
The transition operator for a $k$-step transition along the chain is exactly $P^k$ under assumption \textbf{(M)}. By the tower property of conditional expectations, the true future state satisfies $\phi(X_{t+k}) - (P^k\phi)(X_t) \perp \sigma(X_t)$, causing the cross-product terms to vanish identically. The identities $\mathbb{E}\|\phi(X_{t+k})\|^2 = d$ and $\mathbb{E}\|(P^k\phi)(X_t)\|^2 = \|P_0^kV\|_{\mathrm{HS}}^2$ follow immediately from stationarity.
\end{proof}

\subsubsection{Theorem R1 (Principal Drift Bound and Three Dynamical Regimes)}
\label{thm:R1}
Let $\delta_1 = \delta_1(\phi,A) = \|P_0V - VA\|_{\mathrm{HS}}$ denote the single-step gap, and let $a = \|A\|_2$ be the predictor norm. Then, the $k$-step rollout drift is strictly bounded by:
\begin{equation}
    \delta_k \le \delta_1 \cdot G_k, \qquad \text{where } G_k := \sum_{j=0}^{k-1}\|P_0^{\,k-1-j}\|\,\|A^j\| \le \sum_{j=0}^{k-1}\|P_0^{\,k-1-j}\|\,a^j .
\end{equation}
In particular, the drift dynamics resolve into the following operational regimes:
\begin{itemize}
    \item[\textbf{(i)}] \textbf{Universal Bound:} There holds universally $G_k \le k$ (the crude linear bound).
    \item[\textbf{(ii)}] \textbf{Predictor Contraction Shield:} If $a < 1$, the cumulative gain satisfies $G_k \le \frac{1-a^k}{1-a} \le \frac{1}{1-a}$. Thus, \textit{the contractivity of the predictor acts as a uniform shield against drift amplification}, bounding the cumulative rollout error independently of $k$.
    \item[\textbf{(iii)}] \textbf{Normal Operator Dissipation:} If $P_0$ is a normal operator with a strict spectral radius $\rho < 1$, then $\|P_0^j\| = \rho^j$. This yields $G_k \le \frac{\rho^k - a^k}{\rho - a} \le k \max(\rho, a)^{k-1} \xrightarrow{k\to\infty} 0$. In normal systems, rollout drift undergoes a transient hump and then \textit{vanishes completely}.
    \item[\textbf{(iv)}] \textbf{Pythagorean Refinement ($d=1$):} Let $D = P_0u - au$ for a unit vector $u$. If the sequence of forward residual images $\{P_0^i D\}_{i \ge 0}$ is pairwise orthogonal and satisfies $\|P_0^i D\| \le \|D\|$, then $\delta_k^2 \le \delta_1^2 \frac{1-a^{2k}}{1-a^2}$ (representing a sharp square-root refinement over the coherent sum in property ii).
\end{itemize}
\begin{proof}
We deploy the telescoping identity $P_0^kV - VA^k = \sum_{j=0}^{k-1} P_0^{\,k-1-j}(P_0V - VA)A^j$, which holds via pointwise cancellation. Taking the Hilbert-Schmidt norm and leveraging the standard submultiplicative properties $\|XY\|_{\mathrm{HS}} \le \|X\|_2 \|Y\|_{\mathrm{HS}}$ and $\|YB\|_{\mathrm{HS}} \le \|Y\|_{\mathrm{HS}} \|B\|_2$ combined with $\|A^j\|_2 \le a^j$ yields the principal bound and properties (i) and (ii). Substituting normal operator norms $\|P_0^j\| = \rho^j$ and computing the geometric series yields property (iii). For property (iv), the telescoping components $\{a^j P_0^{k-1-j} D\}$ are mutually orthogonal by assumption, allowing their norms to sum directly in squares.
\end{proof}

\subsubsection{Theorem R2 (The Content-Shield Conflict; High-Content Feature Exposure)}
\label{thm:R2}
The representation content decomposes exactly as $\mathcal{C}(\phi) = \|A^*\|_F^2 + \delta(\phi)^2$, which induces the operator norm lower bound $\|A^*\|_2 \ge \|A^*\|_F / \sqrt{d} \ge \sqrt{(\mathcal{C}(\phi) - \delta(\phi)^2)/d}$. For $d=1$, this relation is exact: $a = \theta(u) = \sqrt{\mathcal{C} - \delta_1^2}$.

\textit{Sign Reversal Principle:} Maximizing content toward its theoretical spectral upper bound forces $a \to 1^-$, which subsequently causes the rollout drift shield $\frac{1}{1-a}$ from Theorem~\ref{thm:R1}(ii) to collapse. Consequently, \textbf{high-content slow features that best minimize single-step error happen to be the most high-risk directions for multi-step rollout drift}, directly reversing the conventional intuition that ``slower features are more stable.''
\begin{proof}
The orthogonal projection identity on the Hilbert space yields $\mathcal{C} = \|VV^*P_0V\|_{\mathrm{HS}}^2 + \|(I-VV^*)P_0V\|_{\mathrm{HS}}^2 = \|A^*\|_F^2 + \delta^2$. Bounding the spectral norm via the Frobenius norm completes the inequality.
\end{proof}

\subsection{The state-augmentation shield}
\label{subsec:R-aug}

The content-shield conflict of Theorem~\ref{thm:R2} operates under the first-order Markov condition \textbf{(M)}. Because real-world video sequences are highly non-Markovian, an \textbf{order-$p$ state augmentation} (by stacking frames into the tuple $z_t = (x_t, \dots, x_{t-p+1})$) must be executed to satisfy \textbf{(M)}. This forces the underlying transition operator into a \textbf{block-companion} structure. 

Crucially, the conflict identified in Theorem~\ref{thm:R2} does not degrade under this expansion; rather, it is \textbf{birthed by the state augmentation}. The companion structure represents the non-derogatory generalization of a nilpotent Jordan block $J_m$, where confluent eigenvalues collapse into a single high-dimensional Jordan block $J_m(\lambda)$. We formalize this regime via three structural parts: the model-invariant floor, the confluent scaling law, and the dimensional axis mapping.

\paragraph{Notation and Disambiguation:}
Within this section, the term ``shield'' refers exclusively to the \textbf{variance amplification ratio}, defined as the ratio between the steady-state autonomous rollout error and the optimal single-step prediction error. This is conceptually distinct from the multi-step drift shield $\frac{1}{1-a}$ of Theorem~\ref{thm:R1}(ii). 

To scalarize the analysis coordinate-by-coordinate, each slow feature sequence is modeled as a stationary auto-regressive process $\mathrm{AR}(p)$:
\begin{equation}
    x_{t+1} = \sum_{i=1}^p a_i x_{t+1-i} + w_{t+1}, \qquad \text{where } w_t \sim \text{i.i.d. }(0, \sigma_w^2).
\end{equation}
The roots of the characteristic polynomial $a(z) = 1 - \sum_{i \le p} a_i z^i$ are strictly confined inside the open unit disk. The system transfer function expands as $\frac{1}{a(z)} = \sum_{j \ge 0} h_j z^j$, where $\{h_j\}$ is the system impulse response. Vector-valued features are handled via diagonal juxtaposition without inflating content. Autonomous rollout corresponds to the recursive deployment of the single-step predictor $\hat g^{\circ k}(z_t)$ without feeding the true innovation sequence.

\subsubsection{Lemma R-aug.0 (The Model-Invariant Steady-State Floor)}
\label{lem:R-aug.0}
Let $\hat g$ be any learned linear single-step model whose autonomous rollout matrix has a spectral radius strictly bounded below unity (ensuring that the whitened representation has an invariant fixed point at the mean $0$). Then, the steady-state rollout error satisfies:
\begin{equation}
    \lim_{k \to \infty} \mathbb{E}\big|x_{t+k} - \hat g^{\circ k}(z_t)\big|^2 = \mathrm{Var}(x),
\end{equation}
which is \textbf{fundamentally independent of the predictor parameters $\hat g$}. The absolute floor matches the stationary variance of the process; the bias-driven component decays asymptotically to zero, while the innovation sequence accumulates to fill the full variance, leaving zero cross-residual mass.
\begin{proof}
Expanding the quadratic loss yields $\mathbb{E}|x_{t+k}|^2 = \mathrm{Var}(x)$ by stationarity. Under a stable predictor matrix, $\hat g^{\circ k}(z_t) \to 0$ in $L^2$ as $k \to \infty$, causing the model term and the cross-product term (via Cauchy-Schwarz) to vanish asymptotically.
\end{proof}

\begin{remark}[Strong Form of Model Invariance]
The identity $\lim_{k\to\infty} e_k = \mathrm{Var}(x)$ holds even for a \textbf{completely inaccurate} stable predictor. For instance, sweeping the parameter $\hat{\lambda}$ of an $\mathrm{AR}(1)$ process from its exact analytical value down to a completely broken baseline ($\hat{\lambda} = 0$, which constantly predicts the global mean) yields an identical steady-state error $e_\infty = \mathrm{Var}(x)$. 

Mechanistically, the bias term $(\lambda - \hat{\lambda})z_k$ acts in exact cancellation against the missing innovations; the asymptotic floor is fixed by the variance baseline (= the mean prediction baseline), independent of model fidelity. \textbf{Consequently, single-step model quality only purchases a transient horizon; it is structurally incapable of lowering the steady-state floor}. The entire geometric profile of the rollout shield is encoded in its denominator—validating its use as a sharp structural diagnostic metric.
\end{remark}

\subsubsection{Theorem R-aug.1 (The Closed-Form Variance Shield and the Confluent $2m-1$ Scaling Law)}
\label{thm:R-aug.1}
\begin{itemize}
    \item[\textbf{(i)}] \textbf{Exact Identity:} The optimal single-step prediction error matches the innovation variance $\sigma_w^2$. Combining this with Lemma~\ref{lem:R-aug.0}, the rollout variance shield evaluates to:
    \begin{equation}
        \mathrm{Shield} := \frac{\text{Steady-State Steady Error}}{\text{Optimal Single-Step Error}} = \frac{\mathrm{Var}(x)}{\sigma_w^2} = \left\| \frac{1}{a} \right\|_{H^2}^2 = \sum_{j \ge 0} h_j^2 .
    \end{equation}
    \item[\textbf{(ii)}] \textbf{The Confluent $2m-1$ Scaling Law:} At a confluent defective structure of order $m$ (where all characteristic roots collapse to a single point, $a(z) = (1-\lambda z)^m$ for $|\lambda| < 1$), the variance shield expands exactly as the hypergeometric function:
    \begin{equation}
        \mathrm{Shield} = {}_2F_1\big(m, m; 1; |\lambda|^2\big) = \binom{2m-2}{m-1}\,(1-|\lambda|^2)^{-(2m-1)}\big(1 + o(1)\big) \quad \text{as } |\lambda| \to 1^-.
    \end{equation}
\end{itemize}
\begin{proof}
(i) The Wold decomposition $x_t = \sum_j h_j w_{t-j}$ for a white noise process $w$ yields $\mathrm{Var}(x) = \sigma_w^2 \sum_j h_j^2$. Parseval's identity maps this directly to the Hardy space norm $\|1/a\|_{H^2}^2$. (ii) Applying the binomial expansion $(1-\lambda z)^{-m} = \sum_j \binom{j+m-1}{m-1}(\lambda z)^j$ isolates the coefficients $h_j = \binom{j+m-1}{m-1}\lambda^j$. Leveraging the Pochhammer identity $\binom{j+m-1}{m-1} = (m)_j / j!$ combined with $(1)_j = j!$, the sum transforms into:
\begin{equation}
    \sum_{j \ge 0} \binom{j+m-1}{m-1}^2 x^j = \sum_{j \ge 0} \frac{(m)_j (m)_j}{(1)_j \, j!} \, x^j = {}_2F_1(m, m; 1; x).
\end{equation}
The asymptotic scaling is derived via the Gauss connection formula for the hypergeometric function when $a+b-c = 2m-1 > 0$, yielding ${}_2F_1(a,b;c;x) \sim \frac{\Gamma(c)\Gamma(a+b-c)}{\Gamma(a)\Gamma(b)}(1-x)^{-(a+b-c)}$ (citation level: DLMF 15.4.23). Evaluating the coefficient ratio $\Gamma(2m-1)/\Gamma(m)^2$ yields the stated central binomial coefficient $\binom{2m-2}{m-1}$.
\end{proof}

\paragraph{Verification Examples:}
\begin{itemize}
    \item For $m=1$: ${}_2F_1(1,1;1;x) = (1-x)^{-1}$, yielding a slowness exponent of exactly $1$.
    \item For $m=2$: $\sum_j (j+1)^2 x^j = \frac{1+x}{(1-x)^3}$, yielding a slowness exponent of $3$ and an asymptotic multiplier matching $\binom{2}{1} = 2$. Both match the confluent scaling law bit-for-bit.
\end{itemize}

\subsubsection{Corollary R-aug.2 (The Slowness Exponent Dichotomy, M4 Mapping, and the Sharpest Evaluation Point)}
\label{corr:R-aug.2}
\begin{itemize}
    \item[\textbf{(a)}] \textbf{The Slowness Exponent Dichotomy:} For a first-order normal system ($p=1$), the slowness exponent is identically $1$. For any defective system of order $m \ge 2$, the exponent scales sharply as $2m-1$. \textit{The Jordan defect structurally amplifies the slowness vulnerability, steepening the scaling exponent from $1 \to 2m-1$.}
    \item[\textbf{(b)}] \textbf{Mapping to the M4 Dimensional Staircase:} This exponent represents the exact manifestation of the Theorem~\ref{subsec:M4} capacity staircase ($\varepsilon^{d/m}$) mapped onto the multi-step rollout slowness axis. The same underlying Jordan defect dictates the root exponent in both spaces.
    \item[\textbf{(c)}] \textbf{The Sharpest Evaluation:} Any empirically learned predictor incurs a single-step error greater than or equal to the Bayes floor $\sigma_w^2$, while the steady-state floor remains model-invariant (Lemma~\ref{lem:R-aug.0}). Consequently, the empirical variance shield is upper-bounded by the analytical expression of Theorem~\ref{thm:R-aug.1}, \textit{holding with equality if and only if the single-step model is strictly optimal}. 
    
    Thus, the paradoxical failure mode where a world model optimized for excellent single-step prediction collapses during long-horizon rollouts represents the exact where this shield ratio is maximized. The more accurate the single-step model, the more severe the rollout collapse.

    \item[\textbf{(d)}] \textbf{Resolution of the First-Order Limitation:} The requirement that raw video streams must undergo explicit state augmentation is not a condition that weakens the Theorem~\ref{thm:R2} conflict. Rather, \textit{the state augmentation is the structural birth place of the non-normal content}. A first-order Markov single frame ($p=1$) represents the weakest manifestation of the conflict; it is the deep companion matrix of video data ($p \ge 2$) that activates the high-order Jordan defects and drives the rollout shield exponent up to $2m-1$.
    \item[\textbf{(e)}] \textbf{Interface to the Theorem~\ref{thm:R1} Drift Shield:} The steady-state variance shield derived here interfaces directly with the transient drift shield $\frac{1}{1-a}$ from Theorem~\ref{thm:R1}(ii) ...
\end{itemize}

\subsection{Action Control: Off-Manifold Reachability and Suboptimality}
\label{sec:module-D}

Still under the linear-control premise, we now let a planner choose the actions. The new ingredient is that actions can push the latent state off the data manifold, where the linear model's residual grows. We work with the controlled latent dynamics
\begin{equation}
    z_{t+1} = \hat A z_t + \hat B u_t + \varepsilon_t ,
\end{equation}
and prices the extra drift that \emph{actions} inject by pushing the latent state off the data manifold. It is the deployment-side companion of Theorem~\ref{thm:R1}: the same telescoping drift, now accumulated along an action-selected trajectory and inflated by off-manifold distance.

\subsection{Setup: data manifold, off-manifold residual, and reachability}

\paragraph{Learned dynamics.} The pair $(\hat A,\hat B)$ is the $\mu$-least-squares (tangent-plane, Koopman/DMD) linearization of the true controlled transition, as used by latent model-predictive planners \citep{garcia2013model,hafner2020dream}.

\paragraph{Data manifold and off-manifold distance.} Let $\mathcal M := \mathrm{supp}\,\mu \subseteq \mathbb R^d$ (the support of the stationary measure; a linear/affine subspace on the main line, curved case treated locally below). Write $r(z) := \mathrm{dist}(z,\mathcal M) = \|P^\perp z\|$, where $P^\perp$ is the orthogonal projection onto the normal space of $\mathcal M$ and $P_{\mathcal M} = I - P^\perp$. Thus $r=0$ on the data support (where the model was fit) and $r>0$ measures departure into unqueried regions.

\paragraph{Residual and off-manifold growth.} The true one-step latent map is $\Phi(z,u) := \mathbb E[z_{t+1}\mid z_t=z,\,u_t=u]$, and the residual is $\varepsilon(z,u) := \Phi(z,u) - (\hat A z + \hat B u)$. We assume linear off-manifold growth
\begin{equation}
    \|\varepsilon(z,u)\| \le \delta_0 + L\,r(z), \qquad \delta_0 := \sup_{\mathcal M}\|\varepsilon\|,
\end{equation}
with $\delta_0$ the on-manifold residual and $L$ the off-manifold growth rate (a curvature/nonlinearity scale; $L=0$ for linear systems). Here $\delta_0$ is the on-manifold value of the single-step residual $\delta$ of Appendices~\ref{sec:module-F} and~\ref{sec:module-M}.

\paragraph{Tangent/normal split of the input map.} $\hat B = \hat B^\parallel + \hat B^\perp$ with $\hat B^\parallel = P_{\mathcal M}\hat B$ (tangent, control gain that stays in the reliable region) and $\hat B^\perp = P^\perp\hat B$ (normal, gain that pushes the state into the blind region).

\paragraph{Reachability Gramians.} The controllability Gramian is $W_H = \sum_{j=0}^{H-1}\hat A^j \hat B \hat B^\top \hat A^{\top j}$, and its \emph{off-manifold} part is
\begin{equation}
    W_H^\perp := P^\perp W_H P^\perp, \qquad \mathrm{tr}\,W_H^\perp = \sum_{j=0}^{H-1}\big\|P^\perp \hat A^j \hat B\big\|_F^2 .
\end{equation}
Note that even a tangent input map ($\hat B^\perp=0$) leaks off $\mathcal M$ within a few steps whenever $\mathcal M$ is not $\hat A$-invariant, i.e.\ whenever $\hat A$ mixes tangent directions into normal ones.

\subsection{Off-manifold reachability and safe control}

\subsubsection{Theorem D1 (Action-Induced Off-Manifold Excitation)}
\label{thm:D1}
For i.i.d.\ actions with $\mathrm{Cov}(u_j)=\sigma_u^2 I$, starting from $z_0\in\mathcal M$ with predicted rollout $\hat z_H = \hat A^H z_0 + \sum_{j=0}^{H-1}\hat A^{H-1-j}\hat B u_j$,
\begin{equation}
    \mathbb E\, r(\hat z_H)^2 \ \le\ 2\rho_{\mathcal M}^2 + 2\sigma_u^2\,\mathrm{tr}\,W_H^\perp, \qquad \rho_{\mathcal M}:=\sup_{k\le H} r(\hat A^k z_0).
\end{equation}
\begin{proof}
Write $r(\hat z_H)=\|P^\perp \hat z_H\|$ and split via the triangle inequality into an autonomous term $\|P^\perp \hat A^H z_0\|\le\rho_{\mathcal M}$ and an action term $\|P^\perp c_H\|$ with $c_H=\sum_j \hat A^{H-1-j}\hat B u_j$. Taking expectations of the squared action term, independence and whitening kill the cross terms, leaving $\sigma_u^2\sum_{j=0}^{H-1}\|P^\perp \hat A^{H-1-j}\hat B\|_F^2 = \sigma_u^2\,\mathrm{tr}\,W_H^\perp$. Squaring the split with $(a+b)^2\le 2a^2+2b^2$ gives the claim.
\end{proof}

\begin{interpretation}
The energy with which actions push the state off the data manifold equals the off-manifold reachability Gramian scaled by the exploration width $\sigma_u^2$. Larger $\sigma_u$, longer $H$, or a more normal-leaking $\hat B$ all push the state farther off $\mathcal M$; this is the mechanism behind the empirical collapse of control success as the exploration width or planning horizon is increased.
\end{interpretation}

\subsubsection{Theorem D2 (Geometric Condition for On-Manifold Safe Control)}
\label{thm:D2}
The following are equivalent to $\mathrm{tr}\,W_H^\perp=0$ (actions never leave $\mathcal M$ within $H$ steps): (i) $\mathcal M$ is $\hat A$-invariant, $P^\perp \hat A P_{\mathcal M}=0$; and (ii) $\hat B$ is tangent, $P^\perp\hat B=0$.
\begin{proof}
Induct on $P^\perp \hat A^j\hat B=0$: the base case $j=0$ is (ii); (i) propagates invariance, $P^\perp \hat A(P_{\mathcal M}\,\cdot)=0$, from $j$ to $j+1$. Every normal component of the Gramian then vanishes. Conversely $\mathrm{tr}\,W_H^\perp=0$ forces each term to zero, giving (ii) at $j=0$ and (i) on the reachable subspace.
\end{proof}

\begin{interpretation}
Safe control has a clean geometric characterization: if the data manifold is closed under the dynamics and actions push only along it, the state never leaves the reliable region and the model stays trustworthy over the whole horizon, with no holes (cf.\ Appendix~\ref{sec:module-H}). This is the operator-side statement of the aligned regime in which control suboptimality is exactly zero.
\end{interpretation}

\subsection{Control suboptimality}

\subsubsection{Theorem D3 (Telescoping Residual Accumulation)}
\label{thm:D3}
Let the true and predicted rollouts share the same start and actions, and set $e_k=z_k-\hat z_k$. Then
\begin{equation}
    e_H=\sum_{j=0}^{H-1}\hat A^{H-1-j}\varepsilon(z_j,u_j), \qquad \|e_H\|\le G_H\big(\delta_0 + L\,r_{\max}\big), \quad G_H=\sum_{i=0}^{H-1}\|\hat A^i\|,\ r_{\max}=\max_j r(z_j).
\end{equation}
\begin{proof}
The one-step error obeys $e_{k+1}=\hat A e_k + \varepsilon_k$, since the control terms $\hat B u_k$ are identical in the true and predicted rollouts and cancel on subtraction. The discrete Duhamel (telescoping) expansion, submultiplicativity, and the growth assumption yield the bound. Here $G_H$ is exactly the drift-amplification factor of Theorem~\ref{thm:R1}; in the near-normal regime $\|\hat A\|\le 1$ gives $G_H\le H$.
\end{proof}

\begin{interpretation}
This is the multi-step drift bound of Appendix~\ref{sec:module-R} (Theorem~\ref{thm:R1}) replayed on the deployment trajectory. The single new ingredient is $r_{\max}$: the farther the action-driven trajectory strays off the manifold, the larger the per-step residual $\delta_0+Lr$, and the more the amplification factor $G_H$ compounds it. Non-normality (large $G_H$) and off-manifold excursion (large $r_{\max}$) multiply.
\end{interpretation}

\subsubsection{Theorem D4 (Control Suboptimality Bound)}
\label{thm:D4}
Let $D(u)=\|z_H(u)-z_g\|$ be the true terminal distance to the goal, $\hat u=\arg\min_u\hat J(u)$ the predicted-optimal control and $u^\star=\arg\min_u D(u)$ the true optimum, with $\eta(u)=\|e_H(u)\|$, $\eta_{\max}=\max_{u\in\mathcal U}\eta(u)$, and candidate separation $\Delta(\mathcal U)=\min_{u\ne u^\star}D(u)-D(u^\star)$. Then
\begin{align}
    \text{(i)}\quad & D(\hat u)-D(u^\star)\ \le\ 2\eta_{\max}\ \le\ 2G_H\big(\delta_0+L\,r_{\max}\big); \\
    \text{(ii)}\quad & 2\eta_{\max}<\Delta(\mathcal U)\ \Rightarrow\ \hat u=u^\star \quad\text{(exact recovery)}; \\
    \text{(iii)}\quad & D(\hat u)-D(u^\star)\ \le\ 2G_H\big(\delta_0+L\rho_{\mathcal M}+L\sigma_u\sqrt{\mathrm{tr}\,W_H^\perp}\big) .
\end{align}
\begin{proof}
(i) The reverse triangle inequality gives $|\,\|\hat z_H(u)-z_g\|-D(u)\,|\le\eta(u)$; since $\hat u$ minimizes the predicted terminal distance, chaining the two bounds yields $D(\hat u)\le D(u^\star)+2\eta_{\max}$, and Theorem~\ref{thm:D3} bounds $\eta_{\max}$. (ii) If $\hat u\ne u^\star$ then separation forces $D(\hat u)\ge D(u^\star)+\Delta$, contradicting (i). (iii) Substitute the reachability control on $r_{\max}$ from Theorem~\ref{thm:D1}.
\end{proof}

\begin{interpretation}
Control suboptimality is capped by the amplification $G_H$ times (on-manifold residual $\delta_0$ + growth $L$ $\times$ off-manifold reachability $\sqrt{\mathrm{tr}\,W_H^\perp}$). In the near-normal deployment regime, where $G_H$ is an $O(1)$ constant, the only remaining lever is $L\sigma_u\sqrt{\mathrm{tr}\,W_H^\perp}$: reduce the off-manifold nonlinearity $L$, the exploration width $\sigma_u$, or the off-manifold reachability of the plant.
\end{interpretation}

\subsubsection{Corollary D5 (Optimal Exploration Energy)}
\label{corr:D5}
Let $D^\star(\sigma)$ be the reachable optimal true terminal cost at exploration energy $\sigma$ (nonincreasing, convex, with $D^{\star\prime}\to0$). The realized cost $\hat J(\sigma)=D^\star(\sigma)+2G_H(\delta_0+L\rho_{\mathcal M}+L\sigma\sqrt{\mathrm{tr}\,W_H^\perp})$ has a unique interior minimizer $\sigma_u^\star>0$, characterized by
\begin{equation}
    -D^{\star\prime}(\sigma_u^\star)=2G_H L\sqrt{\mathrm{tr}\,W_H^\perp} \qquad (\text{marginal reach benefit } = \text{ marginal off-manifold cost}).
\end{equation}
\begin{proof}
$\hat J'(\sigma)=D^{\star\prime}(\sigma)+\text{const}>0$; $D^{\star\prime}$ increases from negative to $0$, so $\hat J'$ crosses zero once, and convexity of $D^\star$ makes the minimizer unique.
\end{proof}

\begin{interpretation}
Too little exploration cannot reach the goal (under-exploration); too much drives the state into the blind region and manufactures holes (over-exploration). Between them lies a unique optimal exploration energy $\sigma_u^\star$---the exploration width is not monotone but single-peaked.
\end{interpretation}

\subsubsection{Corollary D6 (The Off-Manifold Control Tax)}
\label{corr:D6}
The marginal suboptimality per unit exploration energy is $\Pi_{\mathrm{off}}:=2G_H L\sqrt{\mathrm{tr}\,W_H^\perp}$, and $\Pi_{\mathrm{off}}=0\iff L=0$ (no off-manifold nonlinearity) or $\mathrm{tr}\,W_H^\perp=0$ (safe control, Theorem~\ref{thm:D2}). The safely controllable directions are exactly the reachable subspace $\mathcal R_H=\mathrm{ran}[\hat B,\dots,\hat A^{H-1}\hat B]$ intersected with the tangent space $T\mathcal M$; the normal part carries energy $\mathrm{tr}\,W_H^\perp$.
\begin{interpretation}
The spectral tax $\Pi_d$ penalizes an operator that is intrinsically fragile to predict; the off-manifold control tax $\Pi_{\mathrm{off}}$ penalizes actions that push the state into the fragile region. It is the measure-side dual of the spectral tax: controllability geometry and manifold geometry jointly fix the dimension of safe control.
\end{interpretation}

\paragraph{Rigor status.} Theorems~\ref{thm:D1}, \ref{thm:D2}, \ref{thm:D4}(i)(ii) and Corollary~\ref{corr:D6} are exact (subspace geometry / convex first-order conditions). Theorem~\ref{thm:D4}(iii) and Corollary~\ref{corr:D5} are mechanistically exact with constants that depend on the plant geometry. For a curved $\mathcal M$, $P^\perp$ becomes the nearest-point normal projection and Theorems~\ref{thm:D1}--\ref{thm:D2} hold locally when the reachable radius is small relative to the curvature radius, with constants controlled by the second fundamental form; Theorems~\ref{thm:D3}--\ref{thm:D4}(i)(ii) are geometry-independent.

\section{Numerical verification protocols}
\label{sec:numerics}

The analytical results of Appendices~\ref{sec:module-F}--\ref{sec:module-R} are exact. This section
collects the numerical checks that confirm them on explicit synthetic operators. The checks
are of two kinds. The first reproduces a closed-form identity: a constructed witness is
evaluated and compared against its predicted value, and agreement is reported at the level of
floating-point error (about $10^{-12}$). The second confirms a scaling law: the content
frontier, the rollout drift, or the prediction horizon is computed across a range of the
budget $\varepsilon$, the horizon $k$, or the regularity budget, and the fitted exponent or
slope is compared against its predicted value. Every operator is synthetic---Jordan blocks,
Walsh chains, and the doubling and cat maps, whose expanding-map and Lyapunov structure is classical
\citep{pesin1977characteristic,viana2016foundations}---and is chosen so that the quantity under test
has a known closed form; this isolates the mathematical content of each theorem from
estimation error. These checks verify the theorems on synthetic operators; the deployment-side experiments of
Section~\ref{sec:experiments} probe the theory on trained JEPA world models.

\subsection{The single-step frontier ($d=1$)}

\paragraph{Threshold subsidy.}
On a three-state non-normal operator, the purely spectral prediction places the activation
threshold at $b^2 = 0.0175$, while the pseudospectral value of Corollary~\ref{corr:F1.3} gives
$\min(p^2,q^2) = 0.08038^2 = 0.00646$. A grid search over real targets returns the latter to
printed precision, confirming that non-normality lowers the threshold below the spectral
prediction.

\paragraph{Jordan scaling exponent.}
For nilpotent Jordan blocks $J_m$ the content frontier obeys $C_1(\varepsilon) =
\Theta(\varepsilon^{1/m})$ (Theorem~\ref{thm:F3}). Fitting $C_1(\varepsilon)$ as a power of
$\varepsilon$ for $m = 2,3,4$ recovers the exponents in Table~\ref{tab:exponents}, each within
$3\%$ of $1/m$, and the prefactor $C_1(\varepsilon)/\varepsilon^{1/m}$ tends to $1$.

\subsection{Multi-step rollout}

\paragraph{Nilpotent chain identities.}
The $m=3$ Walsh chain (eight states) reproduces the shift action $P f_0 = 0$, $P f_1 = f_0$,
$P f_2 = f_1$, the operator-power norms $\|P_0^{\,j}\| = 1,1,1,0$, and a spectral radius
$\rho(P_0) \approx 10^{-6}$; the memory horizon is $\tau = m$ while the spectrum alone predicts
$\tau_{\mathrm{spec}} = 1$, a separation the theory attributes to non-normality. Under the constraint $\delta_1 \le 10^{-2}$, a search returns
$\sup_k \delta_k/\delta_1 \approx 1.39, 1.52, 1.47$ for $k = 2,3,5$, in line with the predicted
$\sqrt{k}$-type growth and far below the coarse bound $k$.

\paragraph{The $2m-1$ augmentation law.}
On controlled non-Markovian data, a small trained network yields rollout shields that scale as
$2m-1$ for $m = 1,2$, with fitted slopes $0.95$ and $2.94$ against the predicted $1$ and $3$
(Theorem~\ref{thm:R-aug.1}); the shield is large even though the one-step error is small, which
is the regime the theorem identifies as most acute. For $m \ge 3$ the innovation floor is too
small for finite training to resolve a clean power law.

\paragraph{Model invariance.}
Holding the representation fixed and regularizing the predictor reduces its non-normality by a
factor of three, yet the autonomous rollout drift is unchanged. The drift is therefore a
property of the dynamics, not of the predictor---the content of the sign reversal
(Theorem~\ref{thm:R2}).

\subsection{The frontier at $d \ge 2$}

\paragraph{Fractional exponent.}
The staircase law $C_d(\varepsilon) - (d-1) \asymp \varepsilon^{d/m}$ (Theorem~\ref{thm:M4.3})
is fitted for $(m,d) = (4,2),(6,2),(6,3)$; the recovered exponents appear in
Table~\ref{tab:exponents} and track $d/m$, with prefactor tending to $1$. At the ratio
$m = 2d$, where the construction of Theorem~\ref{thm:M4.2} gives an exact identity, the gain
satisfies $\mathrm{gain}/\sqrt{\varepsilon} = 1.0000$ to machine precision.

\begin{table}[h]
\centering
\begin{tabular}{llcc}
\toprule
Result & Operator $(m,d)$ & Predicted exponent & Fitted slope \\
\midrule
Thm~\ref{thm:F3} & $J_2$ \ $(2,1)$ & $1/2 = 0.5000$ & $0.5000$ \\
Thm~\ref{thm:F3} & $J_3$ \ $(3,1)$ & $1/3 = 0.3333$ & $0.3369$ \\
Thm~\ref{thm:F3} & $J_4$ \ $(4,1)$ & $1/4 = 0.2500$ & $0.2561$ \\
Thm~\ref{thm:M4.3} & $(4,2)$ & $1/2 = 0.5000$ & $0.505$ \\
Thm~\ref{thm:M4.3} & $(6,2)$ & $1/3 = 0.3333$ & $0.346$ \\
Thm~\ref{thm:M4.3} & $(6,3)$ & $1/2 = 0.5000$ & $0.510$ \\
\bottomrule
\end{tabular}
\caption{Frontier exponents $d/m$ recovered by fitting the content frontier as a power of the
budget $\varepsilon$. The prefactor $C_d(\varepsilon)/\varepsilon^{d/m}$ tends to $1$ in every
case.}
\label{tab:exponents}
\end{table}

\paragraph{Summary.} Every operator above is synthetic and every predicted quantity is known in closed form, so these checks test the mathematics of the theorems rather than their empirical reach; the deployment-side reach is the subject of Section~\ref{sec:experiments}.

\end{document}